\documentclass{article}

\usepackage[final, nonatbib]{neurips_2022}

\usepackage[utf8]{inputenc} % allow utf-8 input
\usepackage[T1]{fontenc}    % use 8-bit T1 fonts
\usepackage{hyperref}       % hyperlinks
\usepackage{url}            % simple URL typesetting
\usepackage{booktabs}       % professional-quality tables
\usepackage{amsfonts}       % blackboard math symbols
\usepackage{nicefrac}       % compact symbols for 1/2, etc.
\usepackage{microtype}      % microtypography
\usepackage{xcolor}         % colors

\usepackage[utf8]{inputenc}
\usepackage{hyperref}       % hyperlinks
\usepackage{url}            % simple URL typesetting
\usepackage{booktabs}       % professional-quality tables
\usepackage{amsfonts}       % blackboard math symbols
\usepackage{nicefrac}       % compact symbols for 1/2, etc.
\usepackage{microtype}      % microtypography

\usepackage{algorithmic}
\usepackage[linesnumbered,ruled,vlined]{algorithm2e}

\usepackage{graphicx}
\usepackage{subfigure}
\usepackage{booktabs} % for professional tables
\usepackage{amsmath}
\usepackage{amsthm}
\usepackage{float}
\usepackage{bm}
\usepackage{amssymb}
\usepackage{bbm}
\usepackage{color}
\usepackage{mathtools}
\usepackage[super]{nth}
\usepackage{etoolbox}
\usepackage{adjustbox}
\usepackage{enumitem}

\newtheorem{propo}{Proposition}[section]
\newtheorem{lemma}[propo]{Lemma}
\newtheorem{definition}[propo]{Definition}
\newtheorem{coro}[propo]{Corollary}

\newtheorem{theorem}[propo]{Theorem}

\newtheorem{fact}[propo]{Fact}
\newtheorem{remark}[propo]{Remark}

\DeclareMathOperator*{\argmin}{arg\,min}
\def\reals{{\mathbb R}}

\def\eps{\varepsilon}

\def\E{\mathbb E}

\def\u{\mathbf u}
\def\v{\mathbf v}
\def\w{\mathbf{w}}
\def\x{\mathbf{x}}

\DeclareMathOperator{\Poi}{Poi}
\DeclareMathOperator{\Var}{Var}

\title{Trimmed Maximum Likelihood Estimation for Robust Learning in Generalized Linear Models}

\author{%
   Pranjal Awasthi\\
   Google Research \\
   \texttt{pranjalawasthi@google.com}\\
   \And
      Abhimanyu Das\\
   Google Research\\
   \texttt{abhidas@google.com} \\
   \AND
  Weihao Kong \\
  Google Research\\
  \texttt{weihaokong@google.com} \\
   \And
   Rajat Sen \\
   Google Research\\
   \texttt{senrajat@google.com} \\
}

\begin{document}
\maketitle
\begin{abstract}
We study the problem of learning generalized linear models under adversarial corruptions. We analyze a classical heuristic called the \textit{iterative trimmed maximum likelihood estimator} which is known to be effective against \textit{label corruptions} in practice. Under label corruptions, we prove that this simple estimator achieves minimax near-optimal risk on a wide range of generalized linear models, including Gaussian regression, Poisson regression and Binomial regression. Finally, we extend the estimator to the more challenging setting of \textit{label and covariate corruptions} and demonstrate its robustness and optimality in that setting as well.
% setting with slight modification, and showed its robustness and optimality.
\end{abstract}
\section{Introduction}
\label{sec:intro}
Generalized linear models (GLMs) are an elegant framework for statistical modeling of data and are widely used in many applications \cite{nelder1972generalized, dobson2018introduction, mccullagh2019generalized}. A generalized linear model captures the relationship of labels (or observations) $y$ to covariates $\x$ via the conditional density $f(y | \beta^\top \x ) \propto \exp(y \cdot \beta^\top \x - b(\beta^\top \x))$. Here $\beta$ is the parameter of the model. % and $b(\cdot)$ is known as the link function. 
Parameter estimation in GLMs is done via the standard maximum likelihood estimation (MLE) paradigm and has been extensively studied \cite{nelder1972generalized}. While GLMs offer a mathematically tractable formulation for statistical modeling, real data rarely satisfies the generative process of a GLM and as a result there has been considerable interest in developing robust learning algorithms for GLMs \cite{loh2011high, negahban2012unified, prasad2018robust}. The predominant way to model misspecification or adversarial corruptions in the data is Huber's $\epsilon$-contamination model \cite{huber2011robust} and its recent extensions that allow for stronger adversaries \cite{diakonikolas2019robust}. These models assume that a small $\epsilon$ fraction of the data is corrupted by an adversary. Furthermore, the adversary can be restricted to either only corrupting the labels $y$, or be allowed to corrupt both the covariates $\x$ and the labels $y$ for an $\epsilon$ fraction of the data.

Various algorithms have been proposed for robust estimation in generalized linear models. One line of work proposed computationally intractable algorithms that are based on non-convex M estimators \cite{huber2011robust, loh2011high} or by running tournaments over an exponentially large search space \cite{yatracos1985rates}. Another line of work proposes polynomial time algorithms that either achieve sub-optimal error rates \cite{prasad2018robust} or only apply to restricted settings such as the noise being heavy tailed \cite{zhu2021taming}. In practical settings a simple heuristic namely the {\em iterative trimmed estimator} has been shown to work well under settings where only the labels are corrupted \cite{shen2019learning}. However, from a theoretical perspective the iterative trimmed MLE estimator has only been analyzed under restrictive settings such as when the underlying GLM is a Gaussian regression model. 

Our key theoretical contribution in this work is a general analysis of the trimmed MLE estimator. In particular, we show that for a broad family of GLMs, and under adversarial corruptions of only the labels, not only does the iterative trimmed MLE estimator enjoy theoretical guarantees, it in fact nearly achieves the minimax error rate! Next, we also consider the more challenging setting of corruptions to both covariates and labels. In the setting where the covariance of the covariate is known, we leverage the same approach in ~\cite{pensia2020robust} by running the filtering algorithm~\cite{dong2019quantum} as a prepossessing step to trim away the abnormal covariates before applying the iterative trimmed MLE estimator. This can simultaneously handle both covariate and label corruptions and nearly achieve minimax error rates. Below we state our main results.

\begin{theorem}[Informal Theorem]
\label{informal:thm:label}
Let $S_\epsilon = \{(\x_1, y_1), \ldots, (\x_n, y_n)\}$ be independent and identically distributed samples generated by a generalized linear model with sub-Gaussian $x_i$. Let an $\epsilon$ fraction of the labels be adversarially corrupted. Then, with high probability, the iterative trimmed MLE estimator when given as input $S_\epsilon$ provides the following guarantees:
\begin{itemize}
    \item $O(\sigma\epsilon \log(1/\epsilon))$ parameter estimation error for the Gaussian regression model where $\sigma^2$ is the variance of Gaussian noise on $y$.
    \item $O(\epsilon \exp(\sqrt{\log(1/\epsilon)}))$ parameter estimation error for Poisson regression model.
    \item $O(\frac{1}{\sqrt{m}} \epsilon \sqrt{\log(m/\eps)\log(1/\eps)})$ parameter estimation error for the Binomial regression model with $m$ trials.
    \item $O(\epsilon \log(1/\epsilon))$ parameter estimation error for a general class of smooth and continuous GLMs (includes the Gaussian regression model).
\end{itemize}
\end{theorem}

Previous work~\cite{shen2019learning} analyze the iterative trimmed maximum likelihood estimator under the Gaussian regression settings (least square) where only the labels are corrupted, and proved an $O(\sigma)$ $\ell_2$-error bound for parameter estimation. Other iterative trimming approach such as~\cite{bhatia2015robust} also achieves only achieve $O(\sigma)$ error in this setting. Our error bound of $O(\sigma \eps\log(1/\eps))$ significantly improved the dependency on the corruption level $\eps$ and nearly matched the minimax lower bound of $\sigma\eps$~\cite{gao2020robust}. In all the generalized linear models studied in this work, we show that the iterative trimmed maximum likelihood estimator achieves $O(\eps^{1-\delta})$ error for any $\delta>0$, which matches the minimax lower bound $\Omega(\eps)$ up to a sub-polynomial factor.

Next, we present our second main result that can simultaneously handle both covariate and label corruptions and nearly achieve minimax error rate.
\begin{theorem}[Informal Theorem]
\label{informal:thm:covariate}
Let $S_\epsilon = \{(\x_1, y_1), \ldots, (\x_n, y_n)\}$ be independent and identically distributed samples generated by a generalized linear model with sub-Gaussian $x_i$ whose covariance is known. Let an $\epsilon$ fraction of the labels and covariates be adversarially corrupted. After a preprocessing step, the iterative trimmed MLE with high probability achieves the same parameter estimation recovery bounds as in Theorem \ref{informal:thm:label} above.
\end{theorem}
Our algorithm requires the covariance matrix of the covariate distribution being identity or known. Thus the error rate we get is incomparable to the results in the general covariance settings. 

\noindent \textbf{Outline of the paper.} In Section~\ref{sec:related} we discuss related work. We define preliminaries in Section~\ref{sec:prelims} followed by the iterative trimmed MLE algorithm, formal results, and proof sketches for the label corruption case in Section~\ref{sec:main}. In Section~\ref{sec:sample-corruption-model}, we introduce our algorithm, results and proof sketches for the sample corruption case. We defer proofs to the Appendix.
\section{Related Work}
\label{sec:related}
There is a vast amount of literature in statistics, machine learning and theoretical computer science on algorithms that are robust to adversarial corruptions and outliers. Classical works in the Huber's contamination model present algorithms for general robust estimation that obtain near optimal error rates \cite{huber2011robust, yatracos1985rates}. Minimax optimal but computationally inefficient robust estimators have been established in the works of \cite{tukey1975mathematics, yatracos1985rates, chen2015robust, gao2020robust} for a variety of problems such as mean and covariance estimation. In recent years there has also been a line of work in designing computationally efficient algorithms for handling adversarial corruptions \cite{lai2016agnostic, diakonikolas2019robust, charikar2017learning, bhatia2015robust, chen2022online, cherapanamjeri2020optimal, diakonikolas2021outlier}.

Several special cases of generalized linear model has been studied extensively from the robustness perspective. There is a long line of work on designing robust algorithms for the linear least squares problem. This corresponds to the special case of the GLM being a Gaussian regression model \cite{suggala2019adaptive, bhatia2015robust, bhatia2017consistent, klivans2018efficient, chen2013robust, bakshi2021robust, pensia2020robust, diakonikolas2019efficient}. When the covariate follows from a sub-Gaussian distribution, previous best result~\cite{pensia2020robust} achieves $\ell_2$ error $\sigma\eps\sqrt{\log(1/\eps)}$ using Huber regression while we show iterative thresholding achieves $\sigma\eps{\log(1/\eps)}$ error. Another special case of GLMs that has been studied from a robustness perspective is the logistic regression model \cite{feng2014robust, prasad2018robust, chen2020classification}. Our analysis of the iterative trimmed MLE estimator on binomial regression matches the best known error guarantees for this setting. 

The work of~\cite{prasad2018robust} proposes a general procedure for robust gradient descent and shows that one can use this to design robust estimation algorithms for a abroad class of GLMs. Building upon these works, the authors in~\cite{jambulapati2021robust} present a nearly linear time algorithm for GLMs. Compared to our work, \cite{prasad2018robust} assumes $x$ has bounded $8$th moment, \cite{jambulapati2021robust} assumes $x$ is $2$-$4$ hypercontractive, and both papers achieve an $O(\sqrt{\epsilon})$ error guarantee with $\eps$-fraction of corruptions while our algorithm achieves a better $O(\eps)$ guarantee under the stronger sub-Gaussian assumptions. In addition, for GLMs, both papers assume a uniform upper bound (and lower bound) on the second order derivative of function $b(\cdot)$, which is not satisfied by the widely used Poisson and Binomial regression studied in this paper. In a similar setting,~\cite{zhu2021taming} proposed a reweighted MLE estimator for dealing with settings where the covariates are heavy tailed but not corrupted by an adversary.

Iterative thresholding is a longstanding heuristic for robust linear regression that dates back to Legendre~\cite{legendre1959method}. It's theoretical property in the non-asymptotic regime is first studied in~\cite{bhatia2015robust}, which shows a $O(\sigma)$ error bound for sufficiently small label corruption level when the label noise is $N(0, \sigma^2)$. The iterative thresholding algorithm is later extended and analyzed in the \textit{oblivious} label corruption setting~\cite{bhatia2017consistent, suggala2019adaptive}, and is shown  to provide consistent estimate even when the corruption level goes to $1$. ~\cite{pensia2020robust} extended the iterative thresholding algorithm to the heavy-tailed covariate setting and can simultaneously handle both labels and covariate corruptions. In addition,~\cite{pensia2020robust} adapted the implicit result in~\cite{bhatia2017consistent} to show that iterative thresholding algorithm achieves $O(\sigma\sqrt{\eps})$ error rate for sub-Gaussian covariate. \cite{shen2019learning} empirically demonstrated the effectiveness of the trimmed loss estimator under labels corruptions, and also proved theoretical guarantees in a special class of GLMs, which, however, also has $O(\sigma)$ error when specialized to the Gaussian noise setting.~\cite{shen2019iterative} studied the trimmed loss estimator in the mixed linear regression setting. Finally,~\cite{chen2022online} proposed an alternating minimization algorithm for the fixed design linear regression setting with Huber contamination on the labels, which is different from our strong contamination model (with adaptive replacement) for generalized linear model. Their algorithm incorporates a semidefinite programming in the set selection step, and achieves a near-optimal $\tilde{O}(\sigma\eps)$ error. On a very high level, our proof follows from a similar framework as in~\cite{chen2022online}.

\cite{pensia2020robust} first proposed a generic approach to modify an estimator which is robust against label corruption into one that is robust against simultaneous label and covariate corruptions by running a covariate filter algorithm \cite{diakonikolas2020outlier, diakonikolas2019recent} as a preprocessing step. In the Gaussian regression setting, our Algorithm~\ref{alg:alt-sample2} is identical to Algorithm 3 of~\cite{pensia2020robust} and the difference is in the improved error rate. In particular, Lemma 4.1 in~\cite{pensia2020robust} implies a  $O(\sigma\sqrt{\eps})$ error bound while we proved a $O(\sigma\eps\log(1/\eps))$ error bound.
%The work of \cite{prasad2018robust} proposes a general procedure for robust gradient descent and shows that one can use this to design robust estimation algorithms for GLMs. However the bounds obtained have sub-optimal dependence on the error rate $\epsilon$ and also have a mild dimension dependence.

% There is a large body of work both on the theoretical and practical aspects of generalized linear models. Here we discuss the works most relevant to the current setting that we study.  However these algorithms are computationally inefficient as they requires enumeration over an exponentially large search space. Polynomial time algorithms for robust estimation in GLMs have also been proposed.

%In particular, the resulting algorithms cannot handle the case of constant corruption rates. 

\section{Preliminaries}
\label{sec:prelims}
In this section, we formally introduce the robust generalized linear model studied in this paper. First we define the classical generalized linear model as follows.
\begin{definition}[Generalized linear model]\label{def:glm}
We say that $(\x, y)$ follows from a \textit{generalized linear model} if there exist function $b(\cdot)$, and function $c(\cdot)$ such that the probability density function of $y$ equals
$$
f(y|\beta^\top\x) = c(y)\exp({y\cdot \beta^\top\x-b(\beta^\top\x)}).
$$
The derivative $b'(\cdot)$ is called \textit{mean function} as $\E[y|\beta^\top\x] = b'(\beta^\top\x)$. The second order derivative $b''(\cdot)$ is called \textit{variance function} as $\Var[y|\beta^\top\x] = b''(\beta^\top\x)$
\end{definition}
Here we provide three commonly used examples of generalized linear models studied in this paper.
\begin{definition} Commonly used examples of generalized linear model
\begin{itemize}
\item \textbf{Gaussian regression}: 
$b(\theta) = \frac{1}{2}\theta^2$, $c(y) = \frac{\exp(-y^2/2)}{\sqrt{2\pi}}$
\item \textbf{Poisson regression}:
$b(\theta) = \exp(\theta)$, $c(y) = \frac{1}{y!}$
\item \textbf{Binomial regression}: Let $m$ be the number of trials of the binomial distribution. 
$b(\theta) = m\log(1+\exp(\theta))$, $c(y) = \binom{m}{y}$
\end{itemize}
\end{definition}
We consider the sub-Gaussian random design setting in this work, i.e., each covariate $\x_i$ is drawn i.i.d. from a sub-Gaussian distribution with zero mean and covariance $\Sigma$.
\begin{definition}[Sub-Gaussian design]\label{asmp:sub-gaussian-design}
We assume that each covariate $\x_i$ is drawn independently from a zero mean, covariance $\Sigma$ sub-Gaussian distribution with sub-Guassian norm $1$, namely, $\E[\x_i]=0$, $\E[\x_i\x_i^\top] = \Sigma$ and for all $\v \in \reals^d$,
$$
\Pr(\v^\top\x_i\ge t)\le \exp(-t^2).
$$
\end{definition}
Having introduced the generation model of the good data, now we describe the corruption model where the adversary is allowed to corrupt a small fraction of the data points.
\begin{definition}[Corruption model]
We consider two different data generation model with $\eps$ fraction of adversarial corruptions:
\begin{itemize}
    \item \textbf{Label corruption model}: Given $n$ i.i.d. sample $\{(\x_i, y_i)\}_{i=1}^n$  generated by a generalized linear model. The adversary is allowed to inspect the sample, and replace a total of $\eps \cdot n$ labels $y_i$ with arbitrary values.
    \item \textbf{Sample corruption model:} Given $n$ i.i.d. sample $\{(\x_i, y_i)\}_{i=1}^n$ generated by a generalized linear model. The adversary is allowed to inspect the sample, and replace a total of $\eps \cdot n$ data points $(\x_i, y_i)$ with arbitrary values.
\end{itemize}
We call the corrupted dataset $S = T\cup E$ where $T$ contains the set of remaining uncorrupted data points, and $E$ contains the set of data points that is controlled by the adversary.
\end{definition}

The goal of our algorithm is recovering the underlying regression coefficient $\beta^*$ from a set of examples with $\eps$ fraction corrupted under $\ell_2$ error metric, i.e., $\|\hat{\beta}-\beta^*\|$. We assume $\|\beta^*\|\le R$ for a constant $R$, and $\eps<=c$ for a sufficiently small constant $c$. For simplicity of the presentation, throughout this paper, we assume $\Sigma = I_d$. We remark that our algorithm for \textit{label corruption model} applies to the general covariance setting and is able to achieve small estimation error in terms of $\|\Sigma^{1/2}(\hat{\beta}-\beta^*)\|_2$ since we can always (implicitly) whiten the data to apply our analysis (see Section~\ref{sec:non-id-cov} for more details). On the other hand, the algorithm for \textit{sample corruption model} only works with the knowledge of $\Sigma$.
% $\|\hat{\beta}-\beta^*\|$ is small

\section{Label Corruption} 
\label{sec:main}
In this section, we formally describe our algorithm and proof-sketch for the label corruption setting.

\subsection{Algorithm}
\label{sec:algorithm}
We start with defining the trimmed maximum likelihood estimator, which is a simple and natural heuristic for robustly learning generalized linear model.

\begin{definition}[Trimmed maximum likelihood estimator]
Given a set of data points $S = \{(\x_i,y_i)\}_{i=1}^n$, define the \textit{trimmed maximum likelihood estimator} as
$$
    \hat{\beta}(S) = \min_\beta\min_{\hat{S}\subset S,  |\hat{S}|= (1-\eps)n}\sum_{(\x_i, y_i)\in \hat{S}} -\log f(y_i|\beta^\top\x_i)
$$
\end{definition}
In the setting of generalized linear model, the objective of trimmed maximum likelihood estimator is a biconvex problem in $\hat{S}$ and $\beta$ but not jointly convex. The following alternating minimization algorithm is a simple heuristic to approximate the trimmed maximum likelihood estimator whose similar form has been studied in~\cite{bhatia2015robust, bhatia2017consistent, shen2019learning, chen2022online}.

\begin{algorithm}
\caption{Alternating minimization of trimmed maximum likelihood estimator}\label{alg:alt}
\SetKwInput{KwData}{Input}
\SetKwInput{KwResult}{Output}
 \KwData{Set of examples $S = \{(\x_1,y_1),\ldots,(\x_n, y_n)\}$, $\eps$, $\eta$, $R$}
 \KwResult{$\hat{\beta}$}
 $S^{(0)} \gets \argmin_{T\subset [n]:|T|=(1-\eps)n}\sum_{i\in T}|y_i|$\;
 $\hat{\beta}^{(1)} \gets 0$\;
 \For{$t=1$ \KwTo $\infty$ do}{
  Choose  $\hat{S}^{(t)} = \argmin_{T\subset S^{(0)}:|T|=(1-2\eps)n}\sum_{i\in T}-\log f(y_i|\langle\hat{\beta}^{(t)}, \x_i\rangle)$\;
  Compute $\hat{\beta}^{(t+1)} = \argmin_{\beta, \|\beta\|\le R}\sum_{i\in \hat{S}^{(t)}}-\log f(y_i|\langle{\beta}, \x_i\rangle)$\;
  \If{$\frac{1}{n}\sum_{i\in \hat{S}^{(t)}}-\log f(y_i|\langle\hat\beta^{(t+1)}, \x_i\rangle) > \frac{1}{n}\sum_{i\in \hat{S}^{(t)}}-\log f(y_i|\langle\hat{\beta}^{(t)}, \x_i\rangle) - \eta$}{Return $\beta^{(t)}$}
 }
\end{algorithm}
The algorithm starts by naively pruning out $\eps\cdot n$ data points whose labels have the largest magnitude. Then each round of the alternating minimization algorithm has two steps. In optimizing over set $S$, we find the set $\hat{S}^{(t)}$ of size $(1-\eps)n$ with the best likelihood. In optimizing over $\beta$, we find regression coefficient $\beta^{(t)}$ which maximizes the likelihood on the current set of data $\hat{S}^{(t)}$. The algorithm terminates and outputs $\beta$ when the likelihood no longer improves by more than $\eta$. It is clear that the algorithm terminates in $O(1/\eta)$ rounds when the log-likelihood is bounded. Since the original trimmed maximum likelihood estimator is a biconvex optimization problem in $S, \beta$ which is not jointly convex, our algorithm does not guarantee to return a global optimal solution. Nonetheless, as we showed, it does return a first order stationary point which will be close to the true coefficient $\beta^*$. It is worth noting that some recent papers on robust statistics~\cite{cheng2021outlier, cheng2020high, zhurobust} show similar nice statistical properties of an approximate first order stationary point for non-convex optimization problems.

We present the guarantee of our algorithm for Gaussian, Poisson, Binomial regression, and a broad class of generalized linear models.
%The result for other GLMs is deferred to Appendix~\ref{sec:glm}.
\begin{theorem}[Gaussian regression with label corruption]\label{thm:gaussian-label}
Let $S = \{\x_i, y_i\}_{i=1}^n$ be a set of data points generated by a Gaussian regression model with $y_i = \langle\x_i,\beta^*\rangle + \eta_i, \eta_i \sim N(0, \sigma^2)$,  sub-Gaussian design, with $\eps_c$-fraction of label corruption and $n = \Omega(\frac{d+\log(1/\delta)}{\eps^2})$. With probability $1-\delta$, Algorithm~\ref{alg:alt} with parameters $\eps = \eps_c, \eta = \eps_c^2, R=\infty$ terminate within $O(\frac{1}{\min(1,\sigma^2)\eps_c^2})$ iterations, and output an estimate $\hat\beta$ such that
$$
\|\hat\beta-\beta^*\|= O(\sigma\eps_c\log(1/\eps_c))
$$
\end{theorem}

\begin{theorem}[Poisson regression with label corruption]\label{thm:poisson-main}
Let $S = \{\x_i, y_i\}_{i=1}^n$ be a set of data points generated by a Poisson regression model with sub-Gaussian design, with $\eps_c$-fraction of label corruption and $n = \Omega(\frac{d}{\eps^2})$. With probability $0.99$, Algorithm~\ref{alg:alt} with parameters $\eps = 2\eps_c, \eta = \eps_c^2/(dn)$, contant $R\ge \|\beta^*\|$ terminate within $dn/\eps_c^2$ iterations, and output an estimate $\hat\beta$ such that
$$
\|\hat\beta-\beta^*\|= O(\eps_c\exp(\sqrt{\log(1/\eps_c)}))
$$
\end{theorem}

\begin{theorem}[Binomial regression with label corruption]\label{thm:binomial-label}
Let $S = \{\x_i, y_i\}_{i=1}^n$ be generated by a Binomial regression model with sub-Gaussian Design, with $\eps_c$-fraction of label corruption and $n = \Omega(\frac{d+\log(1/\delta)}{\eps^2})$. With probability $1-\delta$, Algorithm~\ref{alg:alt} with parameters $\eps = \eps_c, \eta = \eps_c^2/m$, constant $R\ge \|\beta^*\|$ terminate within $m^2/\eps_c^2$ iterations, and output an estimate $\hat\beta$ such that
$$
\|\hat\beta-\beta^*\|= {O}\left(\eps_c\sqrt{\frac{\log(m/\eps_c)\log(1/\eps_c)}{{m}}}\right)
$$
\end{theorem}
\begin{theorem}[A class of generalized linear model with label corruption]\label{thm:glm-label}
Let $S = \{\x_i, y_i\}_{i=1}^n$ be generated by a generalized linear model with sub-Gaussian Design, with $\eps_c$-fraction of label corruption and $n = \Omega(\frac{d+\log(1/\delta)}{\eps^2})$. Assuming that $C_0\le b''(\cdot)\le C$ for non-zero constants $C_0, C$, $b(0)=0, b'(0)=0$, and $\log(c(y))=O(\log(1/\eps_c)), \forall y\le \Theta(\sqrt{\log(1/\eps_c)})$, then with probability $1-\delta$, Algorithm~\ref{alg:alt} with parameters $\eps = \eps_c, \eta = \eps_c^2, R=\infty$ terminates within $\log(1/\eps_c)/\eps_c^2$ iterations, and output an estimate $\hat\beta$ such that
$$
\|\hat\beta-\beta^*\|= {O}(\eps_c\log(1/\eps_c))
$$
\end{theorem}

\subsection{Proof Sketch}
We provide an intuitive proof sketch for the above theorems (the full proofs are deferred to Section ~\ref{app:label-corruption} in the Appendix). The high level proof framework is similar to~\cite{chen2022online}, although the details are drastically different since the focus of our paper is on random design with strong contamination for a wide range of generalized linear model while~\cite{chen2022online} focuses on linear (least square) regression with Huber corruption on the labels. 

The guarantee of our alternating minimization algorithm relies on two claims: First, the algorithm returns an approximate stationary point $\hat{\beta}$. Second, any approximate stationary point will be close to the true coefficient $\beta^*$. In this section, we will present high level intuition of the proof of the two claims. 

\textbf{Alternating minimization algorithm returns an approximate stationary point.} First we define the first order approximate stationary point as follows. Let $\hat{\beta}\in \reals^d$ be a regression coefficient vector and $\hat{S}$ contains the set of datapoints of size $(1-\eps)n$ with the largest log-likelihood under $\hat\beta$. We call $\hat\beta$ a $\gamma$-approximate stationary point if 
\begin{align*}
\frac{1}{n}\sum_{i\in \hat{S}}\nabla_{\beta} \log f(y_i|\langle\hat\beta, \x_i\rangle)^\top \frac{(\beta^*-\hat\beta)}{\|\beta^*-\hat\beta\|} \le \gamma
\end{align*}
i.e., the gradient of the log-likelihood projected along the $\beta^*-\hat\beta$ direction is small. Our goal is to show when the algorithm terminates, that is when $\hat\beta$ can not be improved by more than $\eta$, the gradient along the $\beta^*-\hat\beta$ must be small.
This is clear where the empirical log-likelihood function is smooth, simply because if the gradient is large, one can improve the log-likelihood by more than $\eta$ which will result in a contradiction. The smoothness (norm of the Hessian matrix) of the empirical log-likelihood in generalized linear model is directly related to the range of $b''(\theta)$. In particular, $b''(\theta)$ is bounded for Gaussian regression and Binomial regression. 

However, a problem arises for Poisson regression where $b''(\theta) = \exp(\theta)$ becomes extremely large for large $\theta$, which results in non-smooth curvature for empirical log-likelihood. We overcome this difficulty by leveraging the special property of function $b(\theta)=\exp(\theta)$ in the Poisson regression setting. Observing that the derivative $b'(\theta)$ is equal to second order derivative $b''(\theta)$ for Poisson regression, the gradient along the $\beta^*-\hat\beta$ direction can not be small when the second order derivative along the $\beta^*-\hat\beta$ direction gets large, which will result in a more than $\eta$ improvement of the log-likelihood by moving toward $\beta^*$ and therefore a contradiction. Hence, the second order derivative along the $\beta^*-\hat\beta$ direction must be small, and we blue have the same argument as in the smooth objective function setting.

\textbf{Any approximate stationary point will be close to the true coefficient.} 
Let us first write down the $\gamma$-approximate stationary condition for generalized linear model as
\begin{align*}
\frac{1}{n}\sum_{i\in \hat{S}}\nabla_{\beta} \log f(y_i|\langle\hat\beta, \x_i\rangle)^\top (\beta^*-\hat\beta)
= \frac{1}{n}\sum_{i\in \hat{S}} (y_i-b'({\hat\beta}^\top\x_i)) (\beta^*-\hat\beta)^\top\x_i \le \gamma\|\beta^*-\hat\beta\|
\end{align*}
Recall that $T$ contains the set of uncorrupted data points, and $E$ contains the set of data points that is controlled by the adversary. Split $\hat{S}$ into $\hat{S}\cap T$ and $\hat{S}\cap E$, and rearrange the terms we get
\begin{align*}
& \frac{1}{n}\sum_{i\in \hat S \cap T}(y_i-b'({\hat\beta}^\top\x_i)) (\beta^*-\hat\beta)^\top\x_i \le - \frac{1}{n}\sum_{i\in \hat S \cap E} (y_i-b'({\hat\beta}^\top\x_i)) (\beta^*-\hat\beta)^\top\x_i + \gamma\|\beta^*-\hat\beta\|.
\end{align*}
To obtain an upper bound on $\|\beta^*-\hat\beta\|$, we will prove a lower bound in terms of $\|\beta^*-\hat\beta\|$ on the left hand side, and an upper bound in terms of $\|\beta^*-\hat\beta\|$ on the right hand side. Finally we will combine the upper and lower bound into an upper bound on $\|\beta^*-\hat\beta\|$.

\textbf{Lower bound on the LHS.}  Note that $\hat{S}\cap T$ contains uncorrupted data points. The high level intuition is that since mean of $y_i$ is $b'({\beta^*}^\top\x_i)$ and $b'(\cdot)$ is monotone, $\left(y_i-b'({\beta^*}^\top\x_i)\right) ({\beta^*}-\hat\beta)^\top\x_i$ should be roughly $O\left(\left(({\beta^*}-\hat\beta)^\top\x_i\right)^2\right)$, and $\left(({\beta^*}-\hat\beta)^\top\x_i\right)^2$ should be proportional to $\|\beta^*-\hat\beta\|^2$ given enough samples. More formally, we will decompose the LHS as
\begin{align*}
&\frac{1}{n} \sum_{i\in \hat S \cap T} (y_i-b'({\hat\beta}^\top\x_i)) (\beta^*-\hat\beta)\x_i\nonumber\\
=& \frac{1}{n} \sum_{i\in \hat S \cap T} \left(y_i-b'({\beta^*}^\top\x_i)\right) ({\beta^*}-\hat\beta)^\top\x_i + \frac{1}{n} \sum_{i\in \hat S \cap T} \left(b'({\beta^*}^\top\x_i) - b'(\hat\beta^\top\x_i) \right)({\beta^*}-\hat\beta)^\top\x_i.
\end{align*}

The first term contains a $(1-\eps)$ fraction of uncorrupted random examples sampled from a zero mean distribution with certain tail bound, e.g. sub-exponential for Gaussian and Binomial regression, $k$-th moment bound for Poisson regression. 

Therefore, we can apply \textit{resilience} property to bound the first term. Overall, we heavily utilize the \textit{resilience} property of the sample set that is drawn from ``nice'' distributions. Take sample mean as an example, \textit{resilience}~\cite{steinhardt2017resilience, zhu2019generalized} (also known as \textit{stability}~\cite{diakonikolas2019recent}) dictates that given a large enough sample set $S=\{\x_i\}_{i=1}^n$, the sample mean of any large enough subset of $S$ will be close to each other. We define mean resilience formally here:
\begin{definition}[Resilience]
Given a sample set $S=\{\x_i\}_{i=1}^n$, suppose for any $T\subset S, |T|\ge (1-\eps) n$, it holds that $\|\frac{1}{|T|}\sum_{i \in T} \x_i - \frac{1}{|S|}\sum_{i \in S} \x_i\|\le \tau$, then we call the set $S$ satisifes $(\eps, \tau)$-resilience.
\end{definition}

 Specifically, under sub-Gaussian distribution, a set $S$ of i.i.d. samples with size $n = \Omega(d/\eps^2)$ satisfies $(\eps, \eps\sqrt{\log(1/\eps)})$ resilience with high probability. Resilience property applies to sub-exponential and $k$-th moment bounded distribution as well, and this gives us a way to control the behavior of any subset of good data.

For the second term, we prove that $\frac{1}{n} \sum_{i\in \hat S \cap T}b(\beta^\top\x_i)$ is a strongly convex function again using the resilience property, which implies 
$\frac{1}{n} \sum_{i\in \hat S \cap T} \left(b'({\beta^*}^\top\x_i) - b'(\hat\beta^\top\x_i) \right)({\beta^*}-\hat\beta)^\top\x_i = \Omega(\|\beta^*-\hat\beta\|^2)$.

\textbf{Upper bound on the RHS.} To upper bound $- \frac{1}{n}\sum_{i\in \hat S \cap E} (y_i-b'({\hat\beta}^\top\x_i)) (\beta^*-\hat\beta)^\top\x_i$, we will prove an upper bound on $\sqrt{\frac{1}{n}\sum_{i\in \hat S \cap E} (y_i-b'({\hat\beta}^\top\x_i))^2}$ and $\sqrt{\frac{1}{n}\sum_{i\in \hat S \cap E} ((\beta^*-\hat\beta)^\top\x_i)^2}$ separately, then apply Cauchy-Schwarz inequality. The key difficulty is bounding $\sqrt{\frac{1}{n}\sum_{i\in \hat S \cap E} (y_i-b'({\hat\beta}^\top\x_i))^2}$, as it contains corrupted data points controlled by an adversary, which does not follow any good property possessed by the good stochastic data. However, since $\hat{S}$ contains $(1-\eps)n$ datapoints with the largest log-likelihood under $\hat\beta$, we can argue that 
$$
\sum_{i\in \hat{S}\cap E}-\log f(y_i|\langle\hat\beta, \x_i\rangle) {\le} \sum_{i\in T\setminus \hat S}-\log f(y_i|\langle\hat\beta, \x_i\rangle)
$$
or even
$$
\max_{i\in \hat{S}\cap E}-\log f(y_i|\langle\hat\beta, \x_i\rangle) {\le} \min_{i\in T\setminus \hat S}-\log f(y_i|\langle\hat\beta, \x_i\rangle)
$$
since otherwise one can replace the data points in $\hat{S}\cap E$ by the ones in $T\setminus\hat{S}$ to form a new set with better likelihood than $\hat{S}$. This gives us an upper bound on the negative log-likelihood of $y_i$ in $\hat{S}\cap E$. Therefore we adopt a two step approach to upper bound the $\sqrt{\frac{1}{n}\sum_{i\in \hat S \cap E} ((\beta^*-\hat\beta)^\top\x_i)^2}$. First we prove an upper bound on the negative log-likelihood on $T\setminus \hat{S}$, which becomes a negative log-likelihood bound on $\hat{S}\cap E$ immediately. Second we turn the negative log-likelihood bound into a square error bound.

The two steps vary drastically for different regression models. For the first step of upper bounding the negative log-likelihood, in Gaussian regression we use the resilience property of the quadratic form of sub-Gaussian random variable. In Poisson regression, we leverage the resilience property of $y_i\x_i$ which is heavy tailed. In Binomial regression, since the distribution only has support size $m$, we directly analyze the resilience of the negative log-likelihood. In general GLMs, due to the generality of the likelihood function, we have to again analyze the resilience of the negative log-likelihood directly. For the second step of turning the log-likelihood bound to a quadratic bound, we get the bound trivially in Gaussian regression since Gaussian likelihood is indeed quadratic. For Poisson and Binomial setting, we have to build a proxy function which lower bound the negative log-likelihood function $-\log f(y_i|\langle\hat\beta, \x_i\rangle)$ to connect it to quadratic function. For the general class of GLMs, we leverage the bounds on the $b''(\cdot)$ and $\log(c(y))$ to obtain a quadratic bound.

\subsection{Proof for Poisson Regression}
As an illustrative example, we show how the above proof sketch can be used for Poisson regression to formally prove Theorem~\ref{thm:poisson-main}. (Proofs for the other GLMs are deferred to Section ~\ref{app:label-corruption} in the Appendix)

\begin{lemma}[Approximate stationary point close to $\beta^*$ for Poisson regression]\label{lemma:poisson-stationary-main}
Given a set of datapoints $S = \{\x_i, y_i\}_{i=1}^n$ generated by a Poisson regression model with $\eps$-fraction of label corruption, and the largest $\eps n$ labels removed.
Let $\hat{\beta}$ be a $\max(\eps, \eps^2/\|\beta^*-\hat\beta\|)$-stationary point and $\|\hat\beta\|\le R$. Given that $n=\Omega(\frac{d}{\eps^2})$, with probability $0.99$, it holds that
$$
\|\hat\beta-\beta^*\| = O(\eps\exp(\Theta(\sqrt{\log(1/\eps)})))
$$
\end{lemma}

Since $b''(\theta) = \exp(\theta)$ is unbounded for Poisson regression, the following lemma (proved in the appendix) shows that alternating minimization algorithm still return an approximate stationary point
\begin{lemma}[Algorithm~\ref{alg:alt} finds an approximate stationary point for Poisson regression]\label{lemma:alt-poisson-main}
Given a set of datapoints $S = \{\x_i, y_i\}_{i=1}^n$ generated by a Poisson model with $\eps_c$-fraction of corruption. Assuming that $n=\Omega(\frac{d+\log(1/\delta)}{\eps^2})$, then with probability $1-\delta$, the output of Algorithm~\ref{alg:alt} with input parameters $\eps=2\eps_c, R\ge \|\beta^*\|, \eta = \eps^2/(dn)$, is a $\max(\eps, \frac{2\eps^2}{\|\beta^*-\hat\beta\|})$-approximate stationary point. \end{lemma}

\begin{proof}[Proof of Theorem~\ref{thm:poisson-main}]
Lemma~\ref{lemma:alt-poisson-main} implies the output of Algorithm~\ref{alg:alt} is a $\max(\eps, \frac{\eps^2}{\|\beta^*-\hat\beta\|})$ approximate stationary point. Lemma~\ref{lemma:poisson-stationary-main} then implies that
$\|\hat\beta-\beta^*\|= O(\eps\exp(\sqrt{\log(1/\eps)}))$. To bound the number of iterations, we need an upper bound on the negative log-likelihood on $\beta = 0$, and a uniform lower bound on the negative log-likelihood. The initial negative log-likelihood is upper bounded by $\frac{1}{n}\sum_{i\in \hat{S}^{(1)}}\log(y_i!)+1 \le O(\E[y_i^2]+1) = O(1)$ where $S^{(1)}$ contains the smallest $(1-\eps)n$ labels. Trivially, there is a $0$ lower bound on the negative log-likelihood for Poisson distribution. Therefore, the algorithm will terminate in $dn/\eps_c^2$ iterations.
\end{proof}

\section{Result for Sample Corruption Model}\label{sec:sample-corruption-model}
The learning problem becomes much harder in the presence of label and covariate corruption, since it is hard to tell whether a data point is corrupted by simply looking at the likelihood of $y_i$. From a technical level, the resilience condition we leveraged on covariate $\x_i$ in set $E$ breaks down when there is covariate corruption. Luckily, we are able to restore the resilience property by first running the filtering algorithm for robust mean estimation~\cite{dong2019quantum}. Specifically, if the covariate distribution has identity (or known) covariance and sub-Gaussian tail, one can apply the filtering algorithm to the $\eps$ corrupted data set, and the resulting data set $\{w_i\x_i\}_{i=1}^n$ will have close to identity covariance and the same resilient condition as an uncorrupted data set. This prepossessing step only takes nearly linear time. This approach is firstly proposed in~\cite{pensia2020robust} as a general method to make an algorithm robust against covariate-corruptions.

\subsection{Algorithm}
% \begin{algorithm}
% \caption{Alternating minimization for trimmed maximum likelihood estimator in sample corruption model}\label{alg:alt-sample}
% \SetKwInput{KwData}{Input}
% \SetKwInput{KwResult}{Output}
%  \KwData{Set of examples $S = \{(\x_1,y_1),\ldots,(\x_n, y_n)\}$, $\eps, \eta, R$}
%  \KwResult{$\hat{\beta}$}
%  $\hat{\beta}^{(1)} \gets 0$\;
%  $S^{(0)} \gets \argmin_{T\subset [n]:|T|=(1-\eps)n}\sum_{i\in T}|y_i|$\;
%  \For{$t=1$ \KwTo $\infty$}{
% %   $Q^{(t)} \gets \argmin_{T\subset S^{(0)}:|T|=(1-2\eps)n}\sum_{i\in T}\langle\hat\beta^{(t)}, \x_i\rangle$\;
%   Let $\hat{\w}^{(t)}$ be the solution of the following SDP
%   \begin{align*}
%   \argmin_{\w} & \sum_{i=1}^n-w_i\log f(y_i|\langle\hat\beta^{(t)},\x_i\rangle)\\
%   \text{s.t.} \;\; & 0 \le w_i \le 1 \;\; \forall i\in S^{(0)}\\
%   & w_i = 0 \;\; \forall i\not\in S^{(0)}\\
%   &\|\w\|_1 = (1-2\eps)n\\
%   & \frac{1}{n}\sum_{i\in [n]} w_i\x_i\x_i^\top \preceq (1+\eps\log(1/\eps))\cdot I
%   \end{align*}
%   Compute $\hat{\beta}^{(t+1)} = \argmin_{\beta, \|\beta\|\le R}\sum_{i=1}^n-\hat{w}_i^{(t)}\log f(y_i|\langle\beta, \x_i\rangle)$\;
%   \If{$\sum_{i=1}^n-\hat{w}_i^{(t)}\log f(y_i|\langle\hat\beta^{(t+1)}, \x_i\rangle) > \sum_{i=1}^n-\hat{w}_i^{(t)}\log f(y_i|\langle\hat{\beta}^{(t)}, \x_i\rangle) - \eta$}{Return $\beta^{(t)}$}
%  }
% \end{algorithm}

\begin{algorithm}
\caption{Alternating minimization of trimmed maximum likelihood estimator in sample corruption model}\label{alg:alt-sample2}
\SetKwInput{KwData}{Input}
\SetKwInput{KwResult}{Output}
 \KwData{Set of examples $S = \{(\x_1,y_1),\ldots,(\x_n, y_n)\}$, $\Sigma$, $\eps$, $\eta$, $R$}
 \KwResult{$\hat{\beta}$}
$ S_0 \gets \{(\Sigma^{-1/2}\x_1,y_1),\ldots,(\Sigma^{-1/2}\x_n, y_n)\}$ \tcp{Whiten the covariates.}
$ S' \gets$ Filtering$(S, \eps)$ \tcp{Algorithm 4 in~\cite{dong2019quantum}}
$\hat\beta\gets $ Algorithm~\ref{alg:alt}$(S', \eps, \eta, R)$\;
 Return $\hat\beta$\;
\end{algorithm}
The guarantee of Algorithm~\ref{alg:alt-sample2} is formalized in the following theorems. 
\begin{theorem}[Gaussian regression with sample corruption]\label{thm:gaussian-sample}
Given a set of datapoints $S = \{\x_i, y_i\}_{i=1}^n$ generated by a Gaussian regression model with $y_i = \langle\x_i,\beta^*\rangle + \eta_i, \eta_i \sim N(0, \sigma^2)$, sub-Gaussian design, $\eps_c$-fraction of sample corruption and $n = \Omega(\frac{d+\log(1/\delta)}{\eps^2})$. With probability $1-\delta$, Algorithm~\ref{alg:alt-sample2} with parameters $\eps = \eps_c, \eta = \eps_c^2, R=\infty$ terminate within $O(\frac{1}{\min(1,\sigma^2)\eps_c^2})$ iterations, and output an estimate $\hat\beta$ such that
$$
\|\hat\beta-\beta^*\|= O(\sigma\eps_c\log(1/\eps_c))
$$
\end{theorem}

\begin{theorem}[Poisson regression with sample corruption]\label{thm:poisson-sample}
Given a set of datapoints $S = \{\x_i, y_i\}_{i=1}^n$ generated by a Poisson regression model with $\eps_c$-fraction of label corruption and $n = \Omega(\frac{d}{\eps^2})$. With probability $0.99$, Algorithm~\ref{alg:alt-sample2} with parameters $\eps = 2\eps_c, \eta = \eps_c^2/(dn), R\ge \|\beta^*\|$ terminate within $dn/\eps_c^2$ iterations, and output an estimate $\hat\beta$ such that
$$
\|\hat\beta-\beta^*\|= O(\eps_c\exp(\sqrt{\log(1/\eps_c)}))
$$
\end{theorem}

\begin{theorem}[Binomial regression with sample corruption]\label{thm:binomial-sample-appendix}
Given a set of datapoints $S = \{\x_i, y_i\}_{i=1}^n$ generated by a Binomial regression model with $\eps_c$-fraction of sample corruption and $n = \Omega(\frac{d+\log(1/\delta)}{\eps^2})$. With probability $1-\delta$, Algorithm~\ref{alg:alt-sample2} with parameters $\eps = \eps_c, \eta = \eps_c^2, R\ge \|\beta^*\|$ terminate within $m/\eps_c^2$ iterations, and output an estimate $\hat\beta$ such that
$$
\|\hat\beta-\beta^*\|= {O}(\eps_c\sqrt{\frac{\log(m/\eps_c)\log(1/\eps_c)}{{m}}})
$$
\end{theorem}
\begin{theorem}[A class of generalized linear model with sample corruption]\label{thm:glm-sample}
Let $S = \{\x_i, y_i\}_{i=1}^n$ be generated by a generalized linear model with sub-Gaussian Design, with $\eps_c$-fraction of sample corruption and $n = \Omega(\frac{d+\log(1/\delta)}{\eps^2})$. Assuming that $C_0\le b''(\cdot)\le C$ for non-zero constants $C_0, C$, $b(0)=0, b'(0)=0$, and $\log(c(y))=O(\log(1/\eps_c)), \forall y\le \Theta(\sqrt{\log(1/\eps_c)})$ With probability $1-\delta$, Algorithm~\ref{alg:alt-sample2} with parameters $\eps = \eps_c, \eta = \eps_c^2, R=\infty$ terminate within $\log(1/\eps_c)/\eps_c^2$ iterations, and output an estimate $\hat\beta$ such that
$$
\|\hat\beta-\beta^*\|= {O}(\eps_c\log(1/\eps_c))
$$
\end{theorem}

The proof is the same compared to the label corruption setting except that since $\x_i, i\in E$ is now controlled by the adversary, we can no longer bound $\sqrt{\frac{1}{n}\sum_{i\in \hat S \cap E} ((\beta^*-\hat\beta)^\top\x_i)^2}$ by the resilience property of (uncorrupted) sub-Gaussian samples. Instead, we will leverage the fact that corrupted sample with small covariance is also resilient.

% \textbf{1. Upper bound on the negative log-likelihood.}
% By the optimality of $\hat S$, it must hold that
% \begin{align*}
% \sum_{i\in \hat{S}\cap E}b({\hat\beta}^\top\x_i) - y_i ({\hat\beta}^\top\x_i) - \log \binom{m}{y_i} {\le} \sum_{i\in T\setminus \hat S}b({\hat\beta}^\top\x_i) -y_i ({\hat\beta}^\top\x_i) - \log \binom{m}{y_i}\\
% \le \sum_{i\in T\setminus \hat S}b({{\beta^*}}^\top\x_i) -y_i ({{\beta^*}}^\top\x_i) - \log \binom{m}{y_i} + \sum_{i\in T\setminus \hat S}(b(\hat\beta^\top\x_i)-y_i(\hat\beta^\top\x_i)) - (b({\beta^*}^\top\x_i)-y_i({\beta^*}^\top\x_i))
% \end{align*}

%i.e. $\E[\left(y_i-b'({\beta^*}^\top\x_i)\right) \x_i] = 0$

\section{Conclusion}\label{sec:conclusion}
In this paper, we provided a general theoretical analysis showing that a simple and practical heuristic namely the iterative trimmed MLE estimator achieves minimax optimal error rates upto a logarithmic factor under adversarial corruptions for a wide class of generalized linear models (GLMs). %We also designed algorithms based on semi-definite programming (SDP) to make the estimator robust to covariate corruptions as well. 
It would also be interesting to study whether our techniques can be extended to design robust algorithms for more general exponential families beyond GLMs.

\bibliographystyle{alpha}
\bibliography{references}

\newcommand{\etalchar}[1]{$^{#1}$}
\begin{thebibliography}{NRWY12}

\bibitem[Ahl22]{ahle2022sharp}
Thomas~D Ahle.
\newblock Sharp and simple bounds for the raw moments of the binomial and
  poisson distributions.
\newblock {\em Statistics \& Probability Letters}, 182:109306, 2022.

\bibitem[BJK15]{bhatia2015robust}
Kush Bhatia, Prateek Jain, and Purushottam Kar.
\newblock Robust regression via hard thresholding.
\newblock {\em Advances in neural information processing systems}, 28, 2015.

\bibitem[BJKK17]{bhatia2017consistent}
Kush Bhatia, Prateek Jain, Parameswaran Kamalaruban, and Purushottam Kar.
\newblock Consistent robust regression.
\newblock {\em Advances in Neural Information Processing Systems}, 30, 2017.

\bibitem[BP21]{bakshi2021robust}
Ainesh Bakshi and Adarsh Prasad.
\newblock Robust linear regression: Optimal rates in polynomial time.
\newblock In {\em Proceedings of the 53rd Annual ACM SIGACT Symposium on Theory
  of Computing}, pages 102--115, 2021.

\bibitem[CAT{\etalchar{+}}20]{cherapanamjeri2020optimal}
Yeshwanth Cherapanamjeri, Efe Aras, Nilesh Tripuraneni, Michael~I Jordan,
  Nicolas Flammarion, and Peter~L Bartlett.
\newblock Optimal robust linear regression in nearly linear time.
\newblock {\em arXiv preprint arXiv:2007.08137}, 2020.

\bibitem[CCM13]{chen2013robust}
Yudong Chen, Constantine Caramanis, and Shie Mannor.
\newblock Robust sparse regression under adversarial corruption.
\newblock In {\em International Conference on Machine Learning}, pages
  774--782. PMLR, 2013.

\bibitem[CDGS20]{cheng2020high}
Yu~Cheng, Ilias Diakonikolas, Rong Ge, and Mahdi Soltanolkotabi.
\newblock High-dimensional robust mean estimation via gradient descent.
\newblock In {\em International Conference on Machine Learning}, pages
  1768--1778. PMLR, 2020.

\bibitem[CDK{\etalchar{+}}21]{cheng2021outlier}
Yu~Cheng, Ilias Diakonikolas, Daniel~M Kane, Rong Ge, Shivam Gupta, and Mahdi
  Soltanolkotabi.
\newblock Outlier-robust sparse estimation via non-convex optimization.
\newblock {\em arXiv preprint arXiv:2109.11515}, 2021.

\bibitem[CGR15]{chen2015robust}
Mengjie Chen, Chao Gao, and Zhao Ren.
\newblock Robust covariance matrix estimation via matrix depth.
\newblock {\em arXiv preprint arXiv:1506.00691}, 2015.

\bibitem[CKMY20]{chen2020classification}
Sitan Chen, Frederic Koehler, Ankur Moitra, and Morris Yau.
\newblock Classification under misspecification: Halfspaces, generalized linear
  models, and evolvability.
\newblock {\em Advances in Neural Information Processing Systems},
  33:8391--8403, 2020.

\bibitem[CKMY22]{chen2022online}
Sitan Chen, Frederic Koehler, Ankur Moitra, and Morris Yau.
\newblock Online and distribution-free robustness: Regression and contextual
  bandits with huber contamination.
\newblock In {\em 2021 IEEE 62nd Annual Symposium on Foundations of Computer
  Science (FOCS)}, pages 684--695. IEEE, 2022.

\bibitem[CSV17]{charikar2017learning}
Moses Charikar, Jacob Steinhardt, and Gregory Valiant.
\newblock Learning from untrusted data.
\newblock In {\em Proceedings of the 49th Annual ACM SIGACT Symposium on Theory
  of Computing}, pages 47--60, 2017.

\bibitem[DB18]{dobson2018introduction}
Annette~J Dobson and Adrian~G Barnett.
\newblock {\em An introduction to generalized linear models}.
\newblock Chapman and Hall/CRC, 2018.

\bibitem[DHL19]{dong2019quantum}
Yihe Dong, Samuel Hopkins, and Jerry Li.
\newblock Quantum entropy scoring for fast robust mean estimation and improved
  outlier detection.
\newblock {\em Advances in Neural Information Processing Systems}, 32, 2019.

\bibitem[DK19]{diakonikolas2019recent}
Ilias Diakonikolas and Daniel~M Kane.
\newblock Recent advances in algorithmic high-dimensional robust statistics.
\newblock {\em arXiv preprint arXiv:1911.05911}, 2019.

\bibitem[DKK{\etalchar{+}}19]{diakonikolas2019robust}
Ilias Diakonikolas, Gautam Kamath, Daniel Kane, Jerry Li, Ankur Moitra, and
  Alistair Stewart.
\newblock Robust estimators in high-dimensions without the computational
  intractability.
\newblock {\em SIAM Journal on Computing}, 48(2):742--864, 2019.

\bibitem[DKP20]{diakonikolas2020outlier}
Ilias Diakonikolas, Daniel~M Kane, and Ankit Pensia.
\newblock Outlier robust mean estimation with subgaussian rates via stability.
\newblock {\em Advances in Neural Information Processing Systems},
  33:1830--1840, 2020.

\bibitem[DKS19]{diakonikolas2019efficient}
Ilias Diakonikolas, Weihao Kong, and Alistair Stewart.
\newblock Efficient algorithms and lower bounds for robust linear regression.
\newblock In {\em Proceedings of the Thirtieth Annual ACM-SIAM Symposium on
  Discrete Algorithms}, pages 2745--2754. SIAM, 2019.

\bibitem[DKSS21]{diakonikolas2021outlier}
Ilias Diakonikolas, Daniel~M Kane, Alistair Stewart, and Yuxin Sun.
\newblock Outlier-robust learning of ising models under dobrushin’s
  condition.
\newblock In {\em Conference on Learning Theory}, pages 1645--1682. PMLR, 2021.

\bibitem[FXMY14]{feng2014robust}
Jiashi Feng, Huan Xu, Shie Mannor, and Shuicheng Yan.
\newblock Robust logistic regression and classification.
\newblock {\em Advances in neural information processing systems}, 27, 2014.

\bibitem[Gao20]{gao2020robust}
Chao Gao.
\newblock Robust regression via mutivariate regression depth.
\newblock {\em Bernoulli}, 26(2):1139--1170, 2020.

\bibitem[Hub11]{huber2011robust}
Peter~J Huber.
\newblock Robust statistics.
\newblock In {\em International encyclopedia of statistical science}, pages
  1248--1251. Springer, 2011.

\bibitem[JLST21]{jambulapati2021robust}
Arun Jambulapati, Jerry Li, Tselil Schramm, and Kevin Tian.
\newblock Robust regression revisited: Acceleration and improved estimation
  rates.
\newblock {\em Advances in Neural Information Processing Systems}, 34, 2021.

\bibitem[JLT20]{jambulapati2020robust}
Arun Jambulapati, Jerry Li, and Kevin Tian.
\newblock Robust sub-gaussian principal component analysis and
  width-independent schatten packing.
\newblock {\em Advances in Neural Information Processing Systems},
  33:15689--15701, 2020.

\bibitem[KKM18]{klivans2018efficient}
Adam Klivans, Pravesh~K Kothari, and Raghu Meka.
\newblock Efficient algorithms for outlier-robust regression.
\newblock In {\em Conference On Learning Theory}, pages 1420--1430. PMLR, 2018.

\bibitem[LRV16]{lai2016agnostic}
Kevin~A Lai, Anup~B Rao, and Santosh Vempala.
\newblock Agnostic estimation of mean and covariance.
\newblock In {\em 2016 IEEE 57th Annual Symposium on Foundations of Computer
  Science (FOCS)}, pages 665--674. IEEE, 2016.

\bibitem[LS59]{legendre1959method}
Adrien-Marie Legendre and DE~Smith.
\newblock On the method of least squares.
\newblock {\em A Source Book in Mathemathics, Ed. DE Smith (originally
  published in 1805)}, pages 576--579, 1959.

\bibitem[LW11]{loh2011high}
Po-Ling Loh and Martin~J Wainwright.
\newblock High-dimensional regression with noisy and missing data: Provable
  guarantees with non-convexity.
\newblock {\em Advances in Neural Information Processing Systems}, 24, 2011.

\bibitem[MN19]{mccullagh2019generalized}
Peter McCullagh and John~A Nelder.
\newblock {\em Generalized linear models}.
\newblock Routledge, 2019.

\bibitem[NRWY12]{negahban2012unified}
Sahand~N Negahban, Pradeep Ravikumar, Martin~J Wainwright, and Bin Yu.
\newblock A unified framework for high-dimensional analysis of $ m $-estimators
  with decomposable regularizers.
\newblock {\em Statistical science}, 27(4):538--557, 2012.

\bibitem[NW72]{nelder1972generalized}
John~Ashworth Nelder and Robert~WM Wedderburn.
\newblock Generalized linear models.
\newblock {\em Journal of the Royal Statistical Society: Series A (General)},
  135(3):370--384, 1972.

\bibitem[PJL20]{pensia2020robust}
Ankit Pensia, Varun Jog, and Po-Ling Loh.
\newblock Robust regression with covariate filtering: Heavy tails and
  adversarial contamination.
\newblock {\em arXiv preprint arXiv:2009.12976}, 2020.

\bibitem[PSBR18]{prasad2018robust}
Adarsh Prasad, Arun~Sai Suggala, Sivaraman Balakrishnan, and Pradeep Ravikumar.
\newblock Robust estimation via robust gradient estimation.
\newblock {\em arXiv preprint arXiv:1802.06485}, 2018.

\bibitem[SBRJ19]{suggala2019adaptive}
Arun~Sai Suggala, Kush Bhatia, Pradeep Ravikumar, and Prateek Jain.
\newblock Adaptive hard thresholding for near-optimal consistent robust
  regression.
\newblock In {\em Conference on Learning Theory}, pages 2892--2897. PMLR, 2019.

\bibitem[SCV17]{steinhardt2017resilience}
Jacob Steinhardt, Moses Charikar, and Gregory Valiant.
\newblock Resilience: A criterion for learning in the presence of arbitrary
  outliers.
\newblock {\em arXiv preprint arXiv:1703.04940}, 2017.

\bibitem[SS19a]{shen2019iterative}
Yanyao Shen and Sujay Sanghavi.
\newblock Iterative least trimmed squares for mixed linear regression.
\newblock {\em Advances in Neural Information Processing Systems}, 32, 2019.

\bibitem[SS19b]{shen2019learning}
Yanyao Shen and Sujay Sanghavi.
\newblock Learning with bad training data via iterative trimmed loss
  minimization.
\newblock In {\em International Conference on Machine Learning}, pages
  5739--5748. PMLR, 2019.

\bibitem[Tuk75]{tukey1975mathematics}
John~W Tukey.
\newblock Mathematics and the picturing of data.
\newblock In {\em Proceedings of the International Congress of Mathematicians,
  Vancouver, 1975}, volume~2, pages 523--531, 1975.

\bibitem[Yat85]{yatracos1985rates}
Yannis~G Yatracos.
\newblock Rates of convergence of minimum distance estimators and kolmogorov's
  entropy.
\newblock {\em The Annals of Statistics}, 13(2):768--774, 1985.

\bibitem[ZJS19]{zhu2019generalized}
Banghua Zhu, Jiantao Jiao, and Jacob Steinhardt.
\newblock Generalized resilience and robust statistics.
\newblock {\em arXiv preprint arXiv:1909.08755}, 2019.

\bibitem[ZJS22]{zhurobust}
Banghua Zhu, Jiantao Jiao, and Jacob Steinhardt.
\newblock Robust estimation via generalized quasi-gradients.
\newblock {\em Information and Inference: A Journal of the IMA},
  11(2):581--636, 2022.

\bibitem[ZZ21]{zhu2021taming}
Ziwei Zhu and Wenjing Zhou.
\newblock Taming heavy-tailed features by shrinkage.
\newblock In {\em International Conference on Artificial Intelligence and
  Statistics}, pages 3268--3276. PMLR, 2021.

\end{thebibliography}

\newpage

\appendix

\section{Label Corruption Proofs}
\label{app:label-corruption}

\subsection{Gaussian}

\begin{proof}[Proof of Theorem~\ref{thm:gaussian-label}]
Applying Lemma~\ref{lemma:alt-smooth} with $\eta = \eps^2$, $C=1/\sigma^2$ implies the output of Algorithm~\ref{alg:alt} is a $O(\eps/\sigma)$ approximate stationary point. Lemma~\ref{lemma:gaussian-stationary} then implies that
$\|\hat\beta-\beta^*\|= O(\sigma\eps\log(1/\eps))$. To bound the number of iterations, we need an upper bound on the negative log-likelihood on $\beta = 0$, and a uniform lower bound on the negative log-likelihood. The initial negative log-likelihood is upper bounded by $\frac{1}{n}\sum_{i\in \hat{S}^{(1)}}y_i^2/\sigma^2$ where $S^{(1)}$ contains the smallest $(1-\eps)n$ labels, which is bounded by $O(\max(1, 1/\sigma^2))$. Trivially, there is a $0$ lower bound on the negative log-likelihood for Gaussian. Therefore, the algorithm will terminate in $\frac{1}{\min(1, \sigma^2)\eps_c^2}$ iterations.
\end{proof}
\begin{lemma}[Approximate stationary point close to $\beta^*$ for Gaussian regression]\label{lemma:gaussian-stationary}
Given a set of datapoints $S = \{\x_i, y_i\}_{i=1}^n$ generated by a Gaussian regression model with $\eps$-fraction of label corruption.
Let $\hat{\beta}$ be a $\eps/\sigma$-stationary point defined in Definition~\ref{def:approx-sp}. Given that $n=\Omega(\frac{d+\log(1/\delta)}{\eps^2})$, with probability $1-\delta$, it holds that
$$
\|\hat\beta-\beta^*\| = O(\sigma\eps\log(1/\eps))
$$
% $$
% \sum_{i\in \hat{S}} (y_i-\hat{\beta}^\top \x_i)\x_i = 0,
% $$
\end{lemma}
\begin{proof}
Let $\hat{S}$ be the set defined in Definition~\ref{def:approx-sp}. 
%If $\|\beta^*-\hat\beta\|\le \sigma\eps$, the proof is done. When $\|\beta^*-\hat\beta\|\ge \sigma\eps$, 
The first order stationary property guarantees
% $$
% \sum_{i\in \hat{S}} (y_i-\hat{\beta}^\top \x_i)\x_i = 0,
% $$
% which implies 
$$
\frac{1}{n}\sum_{i\in \hat{S}}\frac{1}{\sigma^2}(y_i-\hat\beta^\top\x_i)(\beta^*-\hat\beta)^\top\x_i \le \frac{\eps}{\sigma}\|\beta^*-\hat\beta\|
$$
Denote $T=G\setminus L$ as the uncorrupted set of data points. Then we get
\begin{align}
\frac{1}{n}\sum_{i\in \hat{S}\cap T}(y_i-\hat\beta^\top\x_i)(\beta^*-\hat\beta)^\top\x_i \le -\frac{1}{n}\sum_{i\in \hat{S}\cap E}(y_i-\hat\beta^\top\x_i)(\beta^*-\hat\beta)^\top\x_i + \sigma\eps\|\beta^*-\hat\beta\|\label{eqn:gaussian-main}
\end{align}
\textbf{Lower bound on the LHS}

We will first establish a lower bound on the LHS of Equation~\ref{eqn:gaussian-main}, which contains terms from $\hat{S}\cap T$. Note that
\begin{align}
&\frac{1}{n}\sum_{i\in \hat{S}\cap T}(y_i-\hat\beta^\top\x_i)(\beta^*-\hat\beta)^\top\x_i\nonumber\\
=& \frac{1}{n}\sum_{i\in \hat{S}\cap T}(\eta_i+(\beta^*-\hat\beta)^\top\x_i)(\beta^*-\hat\beta)^\top\x_i\nonumber\\
=& \frac{1}{n}\sum_{i\in \hat{S}\cap T}\eta_i(\beta^*-\hat\beta)^\top\x_i +((\beta^*-\hat\beta)^\top\x_i)^2\nonumber\\
\stackrel{\text{resilience}}{\gtrsim}& -\sigma\eps\log(1/\eps)\|\beta^*-\hat\beta\| + \|\beta^*-\hat\beta\|^2(1-\eps\log(1/\eps))\label{eqn:ScapT},
\end{align}
where we have leveraged the following resilience property in Equation~\ref{eqn:ScapT}
\begin{align*}
\|\frac{1}{n}\sum_{i\in \hat{S}\cap T}\eta_i\x_i\|\lesssim \sigma\eps\log(1/\eps)\\
\|\frac{1}{n}\sum_{i\in \hat{S}\cap T}\x_i\x_i^\top-I\|\lesssim \eps\log(1/\eps)
\end{align*}

\textbf{Upper bound on the RHS}

Then we establish an upper bound on the RHS of Equation~\ref{eqn:gaussian-main}, which contains terms from $\hat{S}\cap E$
\begin{align*}
&-\frac{1}{n}\sum_{i\in \hat{S}\cap E}(y_i-\hat\beta^\top\x_i)(\beta^*-\hat\beta)^\top\x_i\\
\stackrel{\text{Cauchy-Schwarz}}{\lesssim}& \left(\frac{1}{n}\sum_{i\in \hat{S}\cap E}(y_i-\hat\beta^\top\x_i)^2 \right)^{1/2} \left(\frac{1}{n}\sum_{i\in \hat{S}\cap E}((\beta^*-\hat\beta)^\top\x_i)^2\right)^{1/2}\\
\stackrel{\text{resilience}}{\lesssim}& \left(\frac{1}{n}\sum_{i\in \hat{S}\cap E}(y_i-\hat\beta^\top\x_i)^2 \right)^{1/2} \eps^{1/2}\log(1/\eps)^{1/2}\|\beta^*-\hat\beta\|
\end{align*}
Observe that since
\begin{align*}
\left(\frac{1}{n}\sum_{i\in \hat{S}}(y_i-\hat\beta^\top\x_i)^2 \right)\stackrel{\text{optimality of }\hat S}{\le} \left(\frac{1}{n}\sum_{i\in {T}}(y_i-\hat\beta^\top\x_i)^2 \right),
\end{align*}
it holds that
\begin{align*}
\left(\frac{1}{n}\sum_{i\in \hat{S}\cap E}(y_i-\hat\beta^\top\x_i)^2 \right)&\le \left(\frac{1}{n}\sum_{i\in {T}}(y_i-\hat\beta^\top\x_i)^2 \right) - \left(\frac{1}{n}\sum_{i\in \hat{S}\cap T}(y_i-\hat\beta^\top\x_i)^2 \right)\\
&= \left(\frac{1}{n}\sum_{i\in T\setminus \hat S}(y_i-\hat\beta^\top\x_i)^2 \right)\\
&=\left(\frac{1}{n}\sum_{i\in T\setminus \hat S}(\eta_i+(\beta^*-\hat\beta)^\top\x_i)^2 \right)\\
&\stackrel{\text{resilience}}{=}O\left(\sigma^2\eps \log(1/\eps) + \|\beta^*-\hat\beta\|^2\eps \log (1/\eps)\right)
\end{align*}
Combining we get
\begin{align}
\frac{1}{n}\sum_{i\in \hat{S}\cap E}(y_i-\hat\beta^\top\x_i)(\beta^*-\hat\beta)^\top\x_i\le \eps\log(1/\eps)(\sigma\|\beta^*-\hat\beta\|+\|\beta^*-\hat\beta\|^2).\label{eqn:gaussian-RHS}
\end{align}
Plugging in Equation~\ref{eqn:ScapT} and~\ref{eqn:gaussian-RHS} into Equation~\ref{eqn:gaussian-main}, we we have
\begin{align*}
-\sigma\eps\log(1/\eps)\|\beta^*-\hat\beta\| + \|\beta^*-\hat\beta\|^2(1-\eps\log(1/\eps))\\
\lesssim \eps\log(1/\eps)(\sigma\|\beta^*-\hat\beta\|+\|\beta^*-\hat\beta\|^2) + \eps\sigma\|\beta^*-\hat\beta\|\\
\implies \eps\log(1/\eps)(\sigma\|\beta^*-\hat\beta\|+\|\beta^*-\hat\beta\|^2)\gtrsim \|\beta^*-\hat\beta\|^2\\
\implies \|\beta^*-\hat\beta\|\lesssim \sigma\eps\log(1/\eps)
\end{align*}
\end{proof}

\subsection{Poisson}

\begin{lemma}[Resilience condition for Poisson regression]\label{lemma:poisson-resilience}
With probability $0.99$ it holds that for all $Q\subset T$ with $|Q|\ge(1-2\eps)n$,
\begin{align*}
\frac{1}{n}\|\sum_{i\in Q} \left(y_i-\exp({\beta^*}^\top\x_i)\right) \x_i\|\lesssim  \eps\exp(\Theta(\sqrt{\log(1/\eps)}))
\end{align*}
and for all $Q\subset T$ with $Q\le \eps n$,
\begin{align*}
\frac{1}{n}\sum_{i\in Q} y_i\lesssim  \eps\exp(\Theta(\sqrt{\log(1/\eps)}))\\
\frac{1}{n}\|\sum_{i\in Q} y_i\x_i\|\lesssim  \eps\exp(\Theta(\sqrt{\log(1/\eps)}))
\end{align*}
\end{lemma}
\begin{proof} $ $\newline
\textbf{Proof of the first statement}\\
We first prove the distribution of
$
\left(y-\exp({\beta^*}^\top\x)\right)\x
$
is $k$-th moment bounded. Note that the $k$-th moment along direction $v$ can be written as

\begin{align*}
&\E[\left(y-\exp({\beta^*}^\top\x)\right)^k (\v^\top\x)^k]\\
\stackrel{\text{Cauchy-Schwarz}}{\le}& \sqrt{\E[\left(y-\exp({\beta^*}^\top\x)\right)^{2k}]}\sqrt{\E[(\v^\top\x)^{2k}]}\\
\end{align*}
Applying the Poisson $k$-th moment bound from Fact~\ref{fact:k-moment-poisson}, and the $k$-th moment bound of sub-Gaussian random variable yields
\begin{align}
\le& \sqrt{\E[\max\{\exp(2k({\beta^*}^\top\x)) , 1\}{2k}^{2k}]} \sqrt{(2k)^{k}}\nonumber\\
\le& \exp(\Theta(k^2\|{\beta^*}\|^2)) k^{\Theta(k)}\label{eqn:poi-gaussian-moment}\\
=& \exp(\Theta(k^2))\nonumber,
\end{align}
where we use the fact that $\E[\exp(\lambda x)]\le \exp(\kappa^2\lambda^2)$ for $\kappa$-sub-Gaussian random variable in Equation~\ref{eqn:poi-gaussian-moment}. Applying Corollary G.1 in~\cite{zhu2019generalized}, we have that with probability $0.99$, $\forall Q\subset T, |Q\ge (1-2\eps)n$,
\begin{align*}
\frac{1}{n}\|\sum_{i\in Q} \left(y_i-\exp(\langle\beta^*,\x_i\rangle)\right) \x_i\|\le  \exp(C k )  (\eps^{1-1/k}+\sqrt{d/n})
\end{align*}
Setting $k = \sqrt{\log(1/\eps)}/$ in the above upper bound yields an upper bound of 
\begin{align*}
\eps\exp(\Theta(\sqrt{\log(1/\eps)}))
\end{align*}

\textbf{Proof of the second statement}\\
Since $\E[y^k]^{1/k} = \exp(\Theta(k))$ 
Applying Corollary G.1 in~\cite{zhu2019generalized} and setting $k = \sqrt{\log(1/\eps)}$, we have that with probability $0.99$, $\forall Q\subset T, |Q\ge (1-\eps)n$,
\begin{align*}
\frac{1}{n}|\sum_{i\in Q} (y_i-\E[y])|\le  \eps\exp(\Theta(\sqrt{\log(1/\eps)}))
\end{align*}

This implies for all $Q\subset T$ with $|Q|\le \eps n$,
\begin{align*}
\frac{1}{n}|\sum_{i\in Q} y_i-\E[y]|\le  \eps\exp(\Theta(\sqrt{\log(1/\eps)}))\\
\stackrel{\E[y]=O(1)}{\implies} \frac{1}{n}|\sum_{i\in Q} y_i|\le  \eps\exp(\Theta(\sqrt{\log(1/\eps)}))
\end{align*}
Setting $k = \sqrt{\log(1/\eps)}/$ in the above upper bound yields an upper bound of $\eps\exp(\Theta(\sqrt{\log(1/\eps)}))$

\textbf{Proof of the third statement}\\
Since we have the same $k$-th moment bound on $y$ and $y-\E[y]$. The bound $\frac{1}{n}\|\sum_{i\in Q} y_i\x_i\|\le  \eps\exp(\Theta(\sqrt{\log(1/\eps)}))$ can be proved similarly.
\end{proof}

\begin{theorem}[Poisson regression with label corruption (Restatement of Theorem~\ref{thm:poisson-main})]\label{thm:poisson-main-restate}
Let $S = \{\x_i, y_i\}_{i=1}^n$ be a set of data points generated by a Poisson regression model with sub-Gaussian design, with $\eps_c$-fraction of label corruption and $n = \Omega(\frac{d}{\eps_c^2})$. With probability $0.99$, Algorithm~\ref{alg:alt} with parameters $\eps = 2\eps_c, \eta = \eps_c^2/(dn), R\ge \|\beta^*\|$ terminate within $dn/\eps_c^2$ iterations, and output an estimate $\hat\beta$ such that
$$
\|\hat\beta-\beta^*\|\lesssim \eps_c\exp(\Theta(\sqrt{\log(1/\eps_c)}))
$$
\end{theorem}
\begin{proof}[Proof of Theorem~\ref{thm:poisson-main-restate}]
Lemma~\ref{lemma:alt-poisson} implies the output of Algorithm~\ref{alg:alt} is a $\max(\eps, \frac{\eps^2}{\|\beta^*-\hat\beta\|})$ approximate stationary point. Lemma~\ref{lemma:poisson-stationary} then implies that
$\|\hat\beta-\beta^*\|= O(\eps\exp(\sqrt{\log(1/\eps)}))$. To bound the number of iterations, we need an upper bound on the negative log-likelihood on $\beta = 0$, and a uniform lower bound on the negative log-likelihood. The initial negative log-likelihood is upper bounded by $\frac{1}{n}\sum_{i\in \hat{S}^{(1)}}\log(y_i!)+1 \le O(\E[y_i^2]+1) = O(1)$ where $S^{(1)}$ contains the smallest $(1-\eps)n$ labels. Trivially, there is a $0$ lower bound on the negative log-likelihood for Poisson distribution. Therefore, the algorithm will terminate in $dn/\eps_c^2$ iterations.
\end{proof}

\begin{lemma}[Approximate stationary point close to $\beta^*$ for Poisson regression (Restatement of Lemma~\ref{lemma:poisson-stationary-main})]\label{lemma:poisson-stationary}
Given a set of datapoints $S = \{\x_i, y_i\}_{i=1}^n$ generated by a Poisson regression model with $\eps$-fraction of label corruption, and the largest $\eps n$ labels removed.
Let $\hat{\beta}$ be a $\max(\eps, \eps^2/\|\beta^*-\hat\beta\|)$-stationary point defined in Definition~\ref{def:approx-sp} and $\|\hat\beta\|\le R$. Given that $n=\Omega(\frac{d}{\eps^2})$, with probability $0.99$, it holds that
$$
\|\hat\beta-\beta^*\| = O(\eps\exp(\Theta(\sqrt{\log(1/\eps)})))
$$
\end{lemma}
\begin{proof}
Recall that $y_i \sim \Poi(\exp({\beta^*}^\top \x_i))$.
Poisson regression log-likelihood: 
$$
\log \Pr(y_i| \langle\beta^*, \x_i\rangle) = y_i ({\beta^*}^\top\x_i)-\exp({\beta^*}^\top\x_i)-\log y_i!
$$
% $$
% \nabla^2 l(S, \beta) = - \sum_{i\in S}\exp(\beta^\top\x_i)\x_i\x_i^\top
% $$
If $\|\beta^*-\hat\beta\|\le \eps$, the proof is done. When $\|\beta^*-\hat\beta\|\ge \eps$, the first order approximate stationary property guarantees
\begin{align*}
&\frac{1}{n}\sum_{i\in S} (y_i-\exp({\hat\beta}^\top\x_i)) (\beta^*-\hat\beta)^\top\x_i \le \eps\|\beta^*-\hat\beta\|\\
\implies& \frac{1}{n}\sum_{i\in \hat S \cap T} (y_i-\exp({\hat\beta}^\top\x_i)) (\beta^*-\hat\beta)^\top\x_i\\
\le& - \frac{1}{n}\sum_{i\in \hat S \cap E} (y_i-\exp({\hat\beta}^\top\x_i)) (\beta^*-\hat\beta)^\top\x_i + \eps\|\beta^*-\hat\beta\|
\end{align*}
\textbf{Lower bound on the LHS}

We will first establish a lower bound on the LHS, which contains terms from $\hat{S}\cap T$. Note that

\begin{align}
&\frac{1}{n}\sum_{i\in \hat S \cap T} (y_i-\exp({\hat\beta}^\top\x_i)) (\beta^*-\hat\beta)^\top\x_i\nonumber\\
=\;& \frac{1}{n}\sum_{i\in \hat S \cap T} \left(y_i-\exp({\beta^*}^\top\x_i)\right) ({\beta^*}-\hat\beta)^\top\x_i\\
&+ \frac{1}{n}\sum_{i\in \hat S \cap T} \left(\exp({\beta^*}^\top\x_i) - \exp(\hat\beta^\top\x_i) \right)({\beta^*}-\hat\beta)^\top\x_i,\label{eqn:poisson-LHS}
\end{align}
and we bound the two terms separately.

Lemma~\ref{lemma:poisson-resilience} implies 
\begin{align}
\frac{1}{n}\sum_{i\in \hat S \cap T} \left(y_i-\exp(\beta^\top\x_i)\right) (\beta^*-\hat\beta)^\top\x_i \le  \|\hat\beta-\beta^*\| \eps\exp(\Theta(\sqrt{\log(1/\eps)}))\label{eqn:poisson-LHS-1}
\end{align}

Now we bound the second term in Equation~\ref{eqn:poisson-LHS}. 
By resilience property (Proposition~\ref{prop:resilience}), the set $L_\beta=\{\beta^\top\x_i<-C\|\beta\|\log(1/\gamma)\}$ satisfies $|L_\beta|\le \gamma n$ for any $\gamma>\eps$, and also it is clear that $\frac{1}{n}\sum_{i\in (\hat S\cap T)\setminus L_\beta}\x_i\x_i^\top \succeq (1- C\gamma\log(1/\gamma))\cdot I$, hence

$$
\frac{1}{n}\sum_{i\in (\hat S\cap T)\setminus L_\beta}\exp(\beta^\top\x_i)\x_i\x_i^\top \succeq \exp(-C\|\beta\|\log(1/\gamma))(1-C\gamma\log(1/\gamma))\cdot I 
$$
Let $\gamma$ be a small constant yields
$$
\frac{1}{n}\sum_{i\in \hat S\cap T}\exp(\beta^\top\x_i)\x_i\x_i^\top\succeq e^{-O(\max(\|\beta\|,1))} I
$$
This implies that 
$\sum_{i\in \hat S\cap T}\exp(\beta^\top\x_i)$ is a strongly convex function in $\beta$ in the range of $\|\beta\|=O(1)$. Since both $\|\beta^*\|=O(1)$ and $\|\hat\beta\|=O(1)$, by the definition of strongly convex function
\begin{align*}
\frac{1}{n}(\nabla \sum_{i\in \hat S\cap T}\exp({\beta^*}^\top\x_i) - \nabla \sum_{i\in \hat S\cap T}\exp(\hat\beta^\top\x_i))({\beta^*}-\hat\beta) \ge \Omega(\|\beta^*-\hat\beta\|^2)\\
\implies \frac{1}{n}(\sum_{i\in \hat S\cap T}\exp({\beta^*}^\top\x_i) -  \sum_{i\in \hat S\cap T}\exp(\hat\beta^\top\x_i))({\beta^*}-\hat\beta)^\top\x_i \ge \Omega(\|\beta^*-\hat\beta\|^2)
\end{align*}

% $$
% \nabla^2 \sum_{i\in \hat S\cap T}\exp(\beta^\top\x_i) = \sum_{i\in \hat S\cap T}\exp(\beta^\top\x_i)\x_i\x_i^\top
% $$

Combining this term with Equation~\ref{eqn:poisson-LHS-1}, we have shown that
\begin{align}
\frac{1}{n}\sum_{i\in \hat S \cap T} (y_i-\exp({\hat\beta}^\top\x_i)) (\beta-\hat\beta)^\top\x_i \ge C_1\|\beta^*-\hat\beta\|^2 - \|\hat\beta-\beta^*\| \eps\exp(C_2\sqrt{\log(1/\eps)})\label{eqn:poisson-LHS-final}
\end{align}
\textbf{Upper bound on the RHS}

% Observe that since
% \begin{align*}
% \frac{1}{n}\sum_{i\in \hat{S}}y_i ({\hat\beta}^\top\x_i)-\exp({\hat\beta}^\top\x_i)-\log y_i!\stackrel{\text{optimality of }\hat S}{\ge} \frac{1}{n}\sum_{i\in T}y_i ({\hat\beta}^\top\x_i)-\exp({\hat\beta}^\top\x_i) - \log y_i!
% \end{align*}
\textbf{1. Upper bound on the negative log-likelihood.} Recall that $|\hat S| = (1-2\eps)n$, $|T| = (1-\eps)n$, and the negative log-likelihood of Poisson regression is $$\exp({\hat\beta}^\top\x_i) - y_i ({\hat\beta}^\top\x_i) + \log y_i!.$$ By the optimality of $\hat S$, it must hold that
\begin{align*}
\max_{i\in \hat{S}\cap E}\exp({\hat\beta}^\top\x_i) - y_i ({\hat\beta}^\top\x_i) + \log y_i! \\{\le} \min_{i\in T\setminus \hat{S}}\exp({\hat\beta}^\top\x_i) -y_i ({\hat\beta}^\top\x_i) + \log y_i\end{align*}
since otherwise one can replace a data point in $\hat{S}\cap E$ by one in $T\setminus\hat{S}$. Since $|T\setminus \hat{S}|\ge \eps n$
\begin{align*}
&\min_{i\in T\setminus\hat{S}}\exp({\hat\beta}^\top\x_i) -y_i({\hat\beta}^\top\x_i) + \log y_i \\
%\le& |\hat{S}\setminus E|\max_{Q \subset T\setminus \hat S, |Q|=\eps n}\min_{i\in Q}\exp({\hat\beta}^\top\x_i) -y_i({\hat\beta}^\top\x_i) + \log y_i\\
\le& \max_{Q\subset T, |Q|=\eps n}\min_{i\in Q}\exp({\hat\beta}^\top\x_i) -y_i ({\hat\beta}^\top\x_i) + \log y_i!\\
\le &\left(\max_{Q\subset T, |Q|=\eps n/3}\min_{i\in Q}\exp({\hat\beta}^\top\x_i) +  \max_{Q\subset T, |Q|=\eps n/3}\min_{i\in Q}-y_i ({\hat\beta}^\top\x_i) + \max_{Q\subset T, |Q|=\eps n/3}\min_{i\in Q}\log y_i!\right)
\end{align*}
Applying Lemma~\ref{lemma:poisson-resilience} and Proposition~\ref{prop:resilience} gives
\begin{align*}
 \le& \left(\exp(\Theta(\sqrt{\log(1/\eps)}))+\max_{Q\subset T, |Q|=\eps n/3}\min_{i\in Q} y_i\log y_i\right)\\
 \le& \left(\exp(\Theta(\sqrt{\log(1/\eps)}))+ \exp(\Theta(\sqrt{\log(1/\eps)}))\Theta(\sqrt{\log(1/\eps)})\right)\\
 \le& \left(\exp(\Theta(\sqrt{\log(1/\eps)}))\right)
\end{align*}
\textbf{2. Turn likelihood bound into square error bound}. Now we have an upper bound on the negative log-likelihood of  a data point in $\hat{S}\cap E$, next step we will turn the log-likelihood bound to a squared error bound. Define proxy function $g_{\hat\beta^\top\x_i}(y_i)$ as 
$$
g_{\hat\beta^\top\x_i}(y_i)  = \exp({\hat\beta}^\top\x_i) - y_i ({\hat\beta}^\top\x_i) + y_i\log y_i - y_i,
$$
which, by Fact~\ref{fact:log-factorial}, is always smaller the negative log-likelihood function
$$
-\log f(y_i|\langle\hat\beta, \x_i\rangle) = \exp({\hat\beta}^\top\x_i) - y_i ({\hat\beta}^\top\x_i) + \log y_i!,
$$
and hence $g_{\hat\beta^\top\x_i}(y_i)\le\exp(\Theta(\sqrt{\log(1/\eps)}))$.

Note that 
\begin{align*}
    g_{\hat\beta^\top\x_i}(\exp(\hat\beta^\top\x_i)) = 0\\ g_{\hat\beta^\top\x_i}'(\exp(\hat\beta^\top\x_i)) = 0\\
    g_{\hat\beta^\top\x_i}'(y_i) = \frac{1}{y_i}.
\end{align*}
Therefore, we can lower bound $g_{\hat\beta^\top\x_i}'(y_i)$ as
$$
g_{\hat\beta^\top\x_i}(y_i) \ge \min(\frac{1}{y_i}, \frac{1}{\exp(\hat\beta^\top\x_i)})(y_i-\exp(\hat\beta^\top\x_i))^2.
$$

Now there are two cases to consider: 1) if $y_i\ge \exp(\hat\beta^\top\x_i)/2$, it holds that
$$
(y_i-\exp(\hat\beta^\top\x_i))^2 \le \sqrt{y_ig_{\hat\beta^\top\x_i}(y_i)} \le 2\exp(\Theta(\sqrt{\log(1/\eps)})),
$$
where we leveraged the fact that $g_{\hat\beta^\top\x_i}(y_i)\le\exp(\Theta(\sqrt{\log(1/\eps)}))$ and $y_i \le \exp(\Theta(\sqrt{\log(1/\eps)}))$ after throwing away the largest $\eps n$ $y_i$s in the beginning of the algorithm. 2) if $y<\exp(\hat\beta^\top\x_i)/2$, we have
\begin{align*}
g_{\hat\beta^\top\x_i}(y_i) \ge  \frac{1}{\exp(\hat\beta^\top\x_i)}(y_i-\exp(\hat\beta^\top\x_i))^2\ge \frac{1}{2}|y_i-\exp(\hat\beta^\top\x_i)|\\
\implies (y_i-\exp(\hat\beta^\top\x_i))^2\le \exp(\Theta(\sqrt{\log(1/\eps)})).
\end{align*}

% $$
% \exp({\hat\beta}^\top\x_i) - y_i ({\hat\beta}^\top\x_i) + y_i\log y_i - y_i 
% $$
% To bound $|y_i-\exp(\hat\beta^\top\x_i)|$, it suffices to consider only $y_i\ge \exp(\hat\beta^\top\x_i)$, and since for $y_i\ge \exp(\hat\beta^\top\x_i)$,

% We need to throw away the largest $\eps$-fraction of $y$ first in the algorithm. We have $|\y_i-\exp(\beta^\top\x_i)|\le \exp(\|\hat\beta\|\sqrt{\log(1/\eps)})+\exp(\|\beta^*\|\sqrt{\log(1/\eps)})$

By Cauchy Schwarz
\begin{align}
&- \frac{1}{n}\sum_{i\in \hat S \cap E} (y_i-\exp({\hat\beta}^\top\x_i))(\beta^*-\hat\beta)^\top\x_i\nonumber\\
\le&\; \sqrt{\frac{1}{n}\sum_{i\in \hat S \cap E} (y_i-\exp({\hat\beta}^\top\x_i))^2}\sqrt{\frac{1}{n}\sum_{i\in \hat S \cap E} ((\beta^*-\hat\beta)^\top\x_i)^2}\nonumber\\
\le&\; \eps^{1/2}\exp\left(\Theta(\sqrt{\log(1/\eps)})\right) \eps^{1/2}\log^{1/2}(1/\eps)\|\beta^*-\hat\beta\|\nonumber\\
 =&\; \eps\exp\left(C_3(\sqrt{\log(1/\eps)})\right)\|\beta^*-\hat\beta\|\label{eqn:poisson-RHS-final}
\end{align}
\textbf{Combining LHS and RHS}. Combining Equation~\ref{eqn:poisson-LHS-final} and Equation~\ref{eqn:poisson-RHS-final} yields
\begin{align*}
    C_1\|\beta^*-\hat\beta\|^2 - \|\hat\beta-\beta^*\| \eps\exp(C_2(\sqrt{\log(1/\eps)})) \\
    \le \eps\exp\left(C_3(\sqrt{\log(1/\eps)})\right)\|\beta^*-\hat\beta\| + \eps\|\beta^*-\hat\beta\|\\
    \implies \|\beta^*-\hat\beta\|\le \eps \exp\left(C_4(\sqrt{\log(1/\eps)})\right)
\end{align*}
\end{proof}

\subsection{Binomial}

\begin{proof}[Proof of Theorem~\ref{thm:binomial-label}]
Lemma~\ref{lemma:alt-smooth} implies the output of Algorithm~\ref{alg:alt} is a $\max(\eps/\sqrt{m}, \frac{\eps^2}{m\|\beta^*-\hat\beta\|})$ approximate stationary point. Lemma~\ref{lemma:binomial-stationary} then implies that
$\|\hat\beta-\beta^*\|= O(\eps\sqrt{\frac{\log(m/\eps)\log(1/\eps)}{{m}}})$. The initial negative log-likelihood is upper bounded by $m$, and trivially, there is a $0$ lower bound on the negative log-likelihood. Therefore, the algorithm will terminate in $m^2/\eps_c^2$ iterations.
\end{proof}

\begin{lemma}[Approximate stationary point close to $\beta^*$ for Binomial regression]\label{lemma:binomial-stationary}
Given a set of datapoints $S = \{\x_i, y_i\}_{i=1}^n$ generated by a Binomial regression model with $m$ trials and $\eps$-fraction of label corruption.
Let $\hat{\beta}$ be an $\max(\eps/\sqrt{m}, \frac{\eps^2}{m\|\beta^*-\hat\beta\|})$-stationary point defined in Definition~\ref{def:approx-sp} with $\|\hat\beta\|\le R$. Given that $n=\Omega(\frac{d+\log(1/\delta)}{\eps^2})$, with probability $1-\delta$, it holds that
$$
\|\beta^*-\hat\beta\|\le {O}(\eps\sqrt{\frac{\log(m/\eps)\log(1/\eps)}{{m}}})
$$
\end{lemma}
\begin{proof}
Recall that the log-likelihood of Binomial regression is  
$$
\log \Pr(y|\langle\beta, \x\rangle) = y\beta^\top\x - m\log(1+\exp(\beta^\top\x))+\log \binom{m}{y}
$$

The following holds for fucntion $b(\beta^\top\x)$ in the Binomial regression.
\begin{align*}
b(\beta^\top\x)& = m\log(1+\exp(\beta^\top\x))\\
\text{Mean function: }\; b'(\beta^\top\x)& = m\frac{1}{1+\exp(-\beta^\top\x)}\\
\text{Variance function: }\; b''(\beta^\top\x)& = m \frac{1}{\left(1+\exp(-\beta^\top\x)\right)\left(1+\exp(\beta^\top\x)\right)}
\end{align*}
The first order approximate stationary property guarantees
\begin{align*}
&\frac{1}{n}\sum_i (y_i-b'({\hat\beta}^\top\x_i)) (\beta^*-\hat\beta)^\top\x_i \le \frac{\eps}{\sqrt{m}}\|\beta^*-\hat\beta\|\\
\implies& \frac{1}{n}\sum_{i\in \hat S \cap T} (y_i-b'({\hat\beta}^\top\x_i)) (\beta^*-\hat\beta)^\top\x_i\\
&\le - \frac{1}{n}\sum_{i\in \hat S \cap E} (y_i-b'({\hat\beta}^\top\x_i)) (\beta^*-\hat\beta)^\top\x_i + \frac{\eps}{\sqrt{m}}\|\beta^*-\hat\beta\|
\end{align*}
\textbf{Lower bound on the LHS}

% Similar to the Poisson setting, the $k$-th moment bound holds
% \begin{align*}
% \E[\max\{b'(\beta^\top\x_i)^k , 1\}k^k(\v^\top\x_i)^k]\\
% \le (Nk^{3/2})^k
% \end{align*}
% Setting $k=\frac{2}{3}\log(1/\eps)$ yields 
% $$
% \sum_{i\in \hat S \cap T} \left(y_i-b'(\beta^\top\x_i)\right) (\beta-\hat\beta)^\top\x_i \le N\eps\log^{3/2}(1/\eps)\|\beta-\hat\beta\|
% $$
We will first establish a lower bound on the LHS, which contains terms from $\hat{S}\cap T$. 
Note that
\begin{align}
\sum_{i\in \hat S \cap T} (y_i-b'({\hat\beta}^\top\x_i)) (\beta^*-\hat\beta)\x_i\nonumber\\
= \sum_{i\in \hat S \cap T} \left(y_i-b'({\beta^*}^\top\x_i)\right) ({\beta^*}-\hat\beta)^\top\x_i + \sum_{i\in \hat S \cap T} \left(b'({\beta^*}^\top\x_i) - b'(\hat\beta^\top\x_i) \right)({\beta^*}-\hat\beta)^\top\x_i,\label{eqn:binomial-LHS}
\end{align}
and we bound the two terms separately. Note that $\frac{y_i-b'(\beta^\top\x_i)}{\sqrt{m}}$ has sub-Gaussian norm $1$, $\x_i$ has sub-Gaussian norm $1$. Let $$\tilde{x}_i = \begin{bmatrix}\x_i\\ (y_i-b'(\beta^\top\x_i))/\sqrt{m}\end{bmatrix}
$$ which is a $1$ sub-Gaussian random vector with mean $0$ and covariance $\tilde{\Sigma} = \begin{bmatrix}I_d&0\\0&c\end{bmatrix}$ for a constant $c$. By Proposition~\ref{prop:resilience} we have that 
$$
\|\frac{1}{n}\sum_{i\in \hat S\cap T}\tilde{\x}_i\tilde{\x}_i^\top-\tilde{\Sigma}\|\le \eps\log(1/\eps).
$$
which implies for $\u = [0, \ldots, 1]^\top\in \reals^{d+1}$ and any $\w = [\v, 0]^\top, \|\v\|=1, \v\in \reals^d$
\begin{align*}
\frac{1}{n\sqrt{m}}\sum_{i\in \hat S \cap T} \left(y_i-b'({\beta^*}^\top\x_i)\right) \v^\top\x_i\\
= \frac{1}{n}\sum_{i\in \hat S\cap T}\u^\top\tilde{\x}_i\tilde{\x}_i^\top\w\le \eps\log(1/\eps).
\end{align*}
Hence,
\begin{align}
\frac{1}{n}\sum_{i\in \hat S \cap T} \left(y_i-b'({\beta^*}^\top\x_i)\right) ({\beta^*}-\hat\beta)^\top\x_i \le \sqrt{m}\eps\log(1/\eps)\|{\beta^*}-\hat\beta\|\label{eqn:binomial-yx-resilience}
\end{align}
Now we bound the second term in Equation~\ref{eqn:binomial-LHS}. By resilience property of sub-Gaussian distribution, for any $\|\beta\| = O(1)$, there is a constant $C$ such that the set $L_\beta=\{|\beta^\top\x_i|>C \log(1/\gamma)\}$ satisfies $|L_\beta|\le \gamma n$ for any $\gamma>\eps$, and also $\sum_{i\in (\hat S\cap T)\setminus L_\beta}\x_i\x_i^\top \succeq (1-C\gamma \log(1/\gamma)) \cdot I$. Hence

$$
\sum_{i\in (\hat S\cap T)\setminus L_\beta}b''(\beta^\top\x_i)\x_i\x_i^\top \succeq m\frac{1}{(1+1/\gamma^C)(1+\gamma^C)}(1-C\gamma\log(1/\gamma))I \succeq \Theta(m\cdot I)
$$
by Setting $\gamma$ as a small constant. This implies $\sum_{i\in (\hat S\cap T)\setminus L_\beta} b(\beta^\top\x_i)$ is a strongly convex function with parameter $\Theta(m)$ when $\|\beta\|=O(1)$. By the definition of strongly convex function, we have 
$$
\sum_{i\in \hat S \cap T} \left(b'(\beta^\top\x_i)-b'(\hat\beta^\top\x_i)\right) (\beta-\hat\beta)^\top\x_i \ge \Omega(m \|\beta-\hat\beta\|^2)
$$
% Second, we show the strong convexity of $\sum_{i\in\hat S\cap T} b(\beta^\top\x_i)$

Combining the two terms, we have shown that
\begin{align}
\frac{1}{n}\sum_{i\in \hat S \cap T} (y_i-b'({\hat\beta}^\top\x_i)) (\beta^*-\hat\beta)^\top\x_i \ge C_1 (m\|\beta^*-\hat\beta\|^2) - C_2(\sqrt{m} \|\hat\beta-\beta^*\| \eps\log(1/\eps) )\label{eqn:binom-lower-bound}
\end{align}

\textbf{Upper bound on the RHS}\\
\textbf{1. Upper bound on the negative log-likelihood.}
By the optimality of $\hat S$, it must hold that
\begin{align*}
&\sum_{i\in \hat{S}\cap E}b({\hat\beta}^\top\x_i) - y_i ({\hat\beta}^\top\x_i) - \log \binom{m}{y_i} {\le} \sum_{i\in T\setminus \hat S}b({\hat\beta}^\top\x_i) -y_i ({\hat\beta}^\top\x_i) - \log \binom{m}{y_i}\\
&\le \sum_{i\in T\setminus \hat S}b({{\beta^*}}^\top\x_i) -y_i ({{\beta^*}}^\top\x_i) - \log \binom{m}{y_i} \\
&+ \sum_{i\in T\setminus \hat S}(b(\hat\beta^\top\x_i)-y_i(\hat\beta^\top\x_i)) - (b({\beta^*}^\top\x_i)-y_i({\beta^*}^\top\x_i))
\end{align*}

The first summation corresponds to the negative log-likelihood of good data under the right model $\beta^*$. The second summation corresponds to the shift in the likelihood from $\beta^*$ to $\hat\beta$. We first bound the first summation.

Define random variable $z_i = -\log f(y_i|\langle\beta^*, \x_i\rangle) = b({\hat\beta}^\top\x_i) - y_i ({\hat\beta}^\top\x_i) - \log \binom{m}{y_i}$. Notice that condition on $\x_i$, $z_i$ can only take $m$ values, hence it is easy to see that $\forall\delta, \Pr(z_i\ge \log(1/\delta)|\x_i)\le m\delta$. Since this is true regardless of $\x_i$, we have 
\begin{align*}
\Pr(z_i\ge \log(1/\delta))\le m\delta\\
\implies \Pr(z_i-\log m \ge t)\le e^{-t}
\end{align*}
Hence $z_i-\log m$ is a 1 sub-exponential random variable. By the resilience property (Corollary~\ref{cor:sub-exponential-resilience}), we have that
$$
\frac{1}{n}\sum_{i\in T\setminus \hat S}b({{\beta^*}}^\top\x_i) -y_i ({{\beta^*}}^\top\x_i) - \log \binom{m}{y_i} = \frac{1}{n}\sum_{i\in T\setminus \hat S} z_i\lesssim \eps\log(m/\eps).
$$
% Notice that the $y_i/m = b'(\hat\beta^\top\x_i)$, where $b'$ is the sigmoid function, minimizes the following quantity:
% \begin{align}
% b({\hat\beta}^\top\x_i) -y_i ({\hat\beta}^\top\x_i) +y_i\log(\frac{y_i}{m})+(m-y_i)\log(\frac{m-y_i}{m})\label{eqn:binom-likelihood}
% \end{align}
% with value $0$.

% The set $|\beta^\top\x_i|\le \sqrt{\log(1/\eps)}$ has size $1-\eps$, each has 
% $b'(\beta^\top\x_i)\ge \exp(-\sqrt{\log(1/\eps)})$. For these $\x_i$s, each has
% $$
% y_i/b'(\beta^\top\x_i)\le 1/2.
% $$
% with probability at most $\exp(-m\exp(-\sqrt{\log(1/\eps)}))$ by Chernoff bound.
% Therefore when $m\ge \log(1/\eps)\exp(\sqrt{\log(1/\eps)})$, $1-\eps$ fraction of them satisfies $y_i/b'(\beta^\top\x_i)\ge 1/2$.
Now we bound the second summation. Note that 
\begin{align*}
&b(\hat\beta^\top\x_i)-y_i(\hat\beta^\top\x_i) \\
\le& b({\beta^*}^\top\x_i)-y_i({\beta^*}^\top\x_i) + (b'({\beta^*}^\top\x_i)-y_i)(\hat\beta-{\beta^*})^\top\x_i + \max b''(\beta^\top\x_i)\cdot  ((\beta^*-\hat\beta)^\top\x_i)^2
\end{align*}

Leveraging Equation~\ref{eqn:binomial-yx-resilience}, resilience of sub-Gaussian sample (Proposition~\ref{prop:resilience}), and $b''(\beta^\top\x_i)\le m, $we get
\begin{align*}
\frac{1}{n}\sum_{i\in T\setminus \hat S}b(\hat\beta^\top)-y_i(\hat\beta^\top\x_i) - \sum_{i\in T\setminus \hat S}b({\beta^*}^\top\x_i)-y_i({\beta^*}^\top\x_i))\\
\lesssim \sqrt{m}\eps\log(1/\eps)\|\beta^*-\hat\beta\|+m\eps\log(1/\eps)\|\beta^*-\hat\beta\|^2
\end{align*}

Combining the two summations yields
\begin{align*}
\frac{1}{n}\sum_{i\in \hat{S}\cap E}b({\hat\beta}^\top\x_i) - y_i ({\hat\beta}^\top\x_i) - \log \binom{m}{y_i}\\ \le \frac{1}{n}\sum_{i\in T\setminus \hat S}b({\hat\beta}^\top\x_i) -y_i ({\hat\beta}^\top\x_i) - \log \binom{m}{y_i}\\
\lesssim \sqrt{m}\eps\log(1/\eps)\|\beta^*-\hat\beta\|+m\eps\log(1/\eps)\|\beta^*-\hat\beta\|^2 + \eps\log(m/\eps)
\end{align*}
\textbf{2. Turn likelihood bound into square error bound}.
From Fact~\ref{fact:log-binom}, define proxy function
\begin{align*}
g_{\hat\beta^\top\x_i}(y_i) = b({\hat\beta}^\top\x_i) -y_i ({\hat\beta}^\top\x_i) +y_i\log(\frac{y_i}{m})+(m-y_i)\log(\frac{m-y_i}{m})+\log\frac{y(m-y)}{m} + C \\
\le b({\hat\beta}^\top\x_i) -y_i ({\hat\beta}^\top\x_i) - \log \binom{m}{y_i}
\end{align*}
% Again there exist $(1-\eps)$ fraction satisfies 
% $$
% (b(\beta^\top\x_i)-y_i(\beta^\top\x_i)) - (b(\hat\beta^\top\x_i)-y_i(\hat\beta^\top\x_i))\le \sqrt{m}\log(1/\eps)\|\beta-\hat\beta\|+m\|\beta-\hat\beta\|^2
% $$
%apply resilience to $\left(b(\hat\beta^\top\x_i)-y_i(\hat\beta^\top\x_i)\right)- \left(b(\beta^\top\x_i)-y_i(\beta^\top\x_i)\right)$, we get Equation~\ref{eqn:binom-likelihood} bounded by $m\|\beta-\hat\beta\|\log(1/\eps)+1$

% The gradient of $y_i$ is 
% $$
% \log(\frac{y_i}{m-y_i}) - \hat\beta^\top\x_i,
% $$
% which is bounded by 
% $$
% \log(m) +|\hat\beta^\top\x_i|,
% $$

% For $1 \le y_i \le m-1$, Equation~\ref{eqn:binom-likelihood} is upper bounded by
% $$
% |y_i- m S(\hat\beta^\top\x_i)| (\log m +|\hat\beta^\top\x_i|).
% $$
% Hence
% \begin{align*}
% &\min_{T\setminus\hat S}|y_i- m S(\hat\beta^\top\x_i)| (\log m +|\hat\beta^\top\x_i|)\\
% \le& \max_{Q\subset T, |Q|=|T\setminus\hat S|/2}\min_{i\in Q}|y_i-mS(\hat\beta^\top\x_i)| \cdot 
% (\log m + \max_{Q\subset T, |Q|=|T\setminus\hat S|/2}\min_{i\in Q}| |\beta^\top\x_i|)\\
% \le& \left(m\|\beta-\hat\beta\|\log(1/\eps) + \sqrt{m}\log(1/\eps)\right)(\log m + \log (1/\eps))
% \end{align*}
% When $y_i=0$, Equation~\ref{eqn:binom-likelihood} is bounded by 
% $$
% b(\hat\beta^\top\x_i)
% $$
Note that 
\begin{align*}
    g_{\hat\beta^\top\x_i}(b'(\hat\beta^\top\x_i)) = C\\ g_{\hat\beta^\top\x_i}'(b'(\hat\beta^\top\x_i)) = 0\\
    g_{\hat\beta^\top\x_i}'(y_i) = \frac{1}{y_i}+\frac{1}{m-y_i} \ge 4/m.
\end{align*}
Combining this with the likelihood bound we get 
$$
\frac{1}{n}\sum_{i\in \hat S\cap E} ({y_i-b'(\hat\beta^\top\x_i)})^2 \lesssim \eps\left( m^2\left(\|\beta^*-\hat\beta\|^2 + \sqrt{1/m}\|\beta^*-\hat\beta\|\right) \log (1/\eps) + (\log m/\eps)m\right)
$$
By Cauchy Schwaz
\begin{align}
- \frac{1}{n}\sum_{i\in \hat S \cap E} (y_i-b'({\hat\beta}^\top\x_i))(\beta-\hat\beta)^\top\x_i\lesssim \sqrt{\frac{1}{n}\sum_{i\in \hat S \cap E} (y_i-b({\hat\beta}^\top\x_i))^2}\sqrt{\frac{1}{n}\sum_{i\in \hat S \cap E} ((\beta-\hat\beta)^\top\x_i)^2}\nonumber\\
\lesssim \eps m\left(
\sqrt{ \|\beta^*-\hat\beta\|^2+\frac{1}{\sqrt{m}}\|\beta^*-\hat\beta\| }\log(1/\eps) + \sqrt{\frac{\log (m/\eps) \log(1/\eps)}{m}} \right)\cdot \|\beta^*-\hat\beta\|\label{eqn:binom-upper-bound}
\end{align}
Combining Equation~\ref{eqn:binom-lower-bound} and Equation~\ref{eqn:binom-upper-bound} yields
% $$
% \|\beta-\hat\beta\|\le \frac{1}{\sqrt{m}}\eps\log(1/\eps) + \eps \log(1/\eps)\sqrt{ \|\beta-\hat\beta\|^2+\frac{1}{\sqrt{m}}\|\beta-\hat\beta\| + \frac{\log m}{m}} 
% $$
% which implies
$$
\|\beta^*-\hat\beta\|\le {O}(\eps\sqrt{\frac{\log(m/\eps)\log(1/\eps)}{{m}}})
$$
\end{proof}

\subsection{Generalized Linear Model}\label{sec:glm}

\begin{theorem}[Generalized linear model with label corruption (Restatement of Theorem~\ref{thm:glm-label})]\label{thm:glm-label-restatement}
Let $S = \{\x_i, y_i\}_{i=1}^n$ be generated by a generalized linear model with sub-Gaussian Design, with $\eps_c$-fraction of label corruption and $n = \Omega(\frac{d+\log(1/\delta)}{\eps^2})$. Assuming that $C_0\le b''(\cdot)\le C$ for non-zero constants $C_0, C$, $b(0)=0, b'(0)=0$, and $\log(c(y))=O(\log(1/\eps)), \forall y\le \Theta(\sqrt{\log(1/\eps)})$ With probability $1-\delta$, Algorithm~\ref{alg:alt} with parameters $\eps = \eps_c, \eta = \eps_c^2, R=\infty$ terminate within $\log(1/\eps_c)/\eps_c^2$ iterations, and output an estimate $\hat\beta$ such that
$$
\|\hat\beta-\beta^*\|= {O}(\eps_c\log(1/\eps_c))
$$
\end{theorem}
\begin{proof}
Lemma~\ref{lemma:alt-smooth} implies the output of Algorithm~\ref{alg:alt} is a $\max(\eps_c, \frac{\eps_c^2}{\|\beta^*-\hat\beta\|})$ approximate stationary point. Lemma~\ref{lemma:glm-stationary} then implies that
$\|\hat\beta-\beta^*\|= O(\eps_c\log(1/\eps_c))$. 
Now we analyze the iteration complexity. Note that Equation~\ref{eqn:glm-RHS-neg-upper} shows that the $\eps$ quantile of negative log-likelihood $-\log f(y_i|\langle\beta^*,\x_i\rangle)$ is upper bounded by $\log(1/\eps)$. Since the algorithm start $\hat\beta$ from $0$, using the fact that $b(0)=0$, the initial negative log-likelihood is upper bounded by 
\begin{align*}
-\log f(y_i|\langle\beta^*,\x_i\rangle) - b(\langle\beta^*,\x_i\rangle) + y_i\langle\beta^*,\x_i\rangle\\
\le -\log f(y_i|\langle\beta^*,\x_i\rangle) + y_i\langle\beta^*,\x_i\rangle.
\end{align*}
Since $y_i\le \sqrt{\log(1/\eps_c)}$ and the $\eps_c$-quantile of $\langle\beta^*,\x\rangle$ is bounded by $\sqrt{\log(1/\eps_c)}$. We get that the negative log-likelihood is upper bounded by $\Theta(\log(1/\eps_C))$.
From Equation~\ref{eqn:glm-neg-lb}, we know that there is a $-\log(1/\eps_c)$ lower bound on the negative log-likelihood. Therefore, the algorithm will terminate in $\log(1/\eps_c)/\eps_c^2$ iterations.
\end{proof}
\begin{remark}
The assumption that $b''(\cdot)\le C$ makes sure the variance of $y$ is bounded, and $b''(\cdot) \ge C_0$ makes sure the distribution of $y$ never degenerate to a singular point. Without these conditions, in the minimax sense, learning $\beta^*$ with finite sample is impossible. The additional condition on $c(y)$ exists to rule out having a single $y_i$ with extremely large density. This would be problematic since the adversary can inject $y_i$ at the point and the trimmed maximum likelihood estimator will not be able to remove these datapoints. Note that this is a mild condition since this essentially only requires $c(y)\le \exp(y^2)$, and it is trivially true for probability mass function. All these assumptions are satisfied by the Gaussian regression model studied in this work.
\end{remark}
\begin{lemma}[Approximate stationary point close to $\beta^*$ for generalized linear regression]\label{lemma:glm-stationary}
Given a set of datapoints $S = \{\x_i, y_i\}_{i=1}^n$ generated by a generalized linear model with $\eps$-fraction of label corruption.
Let $\hat{\beta}$ be an $\max(\eps, \eps^2/\|\beta^*-\hat\beta\|)$-stationary point defined in Definition~\ref{def:approx-sp}.  Assuming that $C_0\le b''(\cdot)\le C$ for non-zero constants $C_0, C$, $b(0)=0, b'(0)=0$, and $\log(c(y))=O(\log(1/\eps)), \forall y\le \Theta(\sqrt{\log(1/\eps)})$. Given that $n=\Omega(\frac{d+\log(1/\delta)}{\eps^2})$, with probability $1-\delta$, it holds that
$$
\|\beta^*-\hat\beta\|\le {O}(\eps\log(1/\eps))
$$
\end{lemma}
\begin{proof}

% Distribution of $y$ under GLM is 
% $$
% \log \Pr(y|\beta, \x) = (y\beta^\top\x - b(\beta^\top\x))+\log c(y)
% $$
The first order approximate stationary property guarantees
\begin{align*}
&\frac{1}{n}\sum_i (y_i-b'({\hat\beta}^\top\x_i)) (\beta^*-\hat\beta)^\top\x_i \le \eps\|\beta^*-\hat\beta\|\\
\implies& \frac{1}{n}\sum_{i\in \hat S \cap T} (y_i-b'({\hat\beta}^\top\x_i)) (\beta^*-\hat\beta)^\top\x_i \\
&\le - \frac{1}{n}\sum_{i\in \hat S \cap E} (y_i-b'({\hat\beta}^\top\x_i)) (\beta^*-\hat\beta)^\top\x_i + \eps\|\beta^*-\hat\beta\|
\end{align*}

\textbf{Lower bound on the LHS}

We first show $y$ is a sub-Gaussian random variable under our assumptions.
\begin{propo}
Suppose a generalized linear model satisfies $b''(\theta)\le C$ for all $\theta\in \reals$, then $(y-\E[y])|\langle\beta^*, \x\rangle$ has sub-Gaussian norm $\sqrt{C}$ for any $\x$
\end{propo}
\begin{proof}

First note that since a probability density function must sum to one, it hold that for any $\theta\in R$,
\begin{align*}
\int c(y)\exp(\theta y-b(\theta))dy = 1\\
\implies \int c(y)\exp(\theta y)dy = \exp(b(\theta)).
\end{align*}
Also recall that $\E[y|\theta] = b'(\theta)$.
Fix $\theta$, the moment generating function of $y-\E[y]$ can be written as
\begin{align*}
\E[e^{t(y-b'(\theta))}] = \int \exp(t(y-b'(\theta)))\cdot c(y)\exp(\theta\cdot y-b(\theta))dy\\
=\int  c(y)\exp((\theta+t)y-b(\theta)-tb'(\theta))dy\\
= \exp(b(\theta+t)-b(\theta)-tb'(\theta)).
\end{align*}
Note that since $b''(\theta) \le C$, it holds that 
$b(\theta+t)\le b(\theta)+b'(\theta)t+ Ct^2$, and therefore
\begin{align}
\exp(b(\theta+t)-b(\theta)-tb'(\theta))\\
\le \exp(Ct^2).
\end{align}
This implies $y-\E[y]|\langle\beta^*, \x\rangle$ is $\sqrt{C}$-sub-Gaussian random variable.
\end{proof}
Note that $\E[y|\x]=b'(\langle\beta^*,\x\rangle)\le C\langle\beta^*,\x\rangle + b'(0)$. Since $\langle\beta^*,\x\rangle$ has sub-Gaussian norm $1$, $\E[y|\x]-b'(0)$ is a $C$-sub-Gaussian random variable. Therefore $y-b'(0)$ is has sub-Gaussian norm $C+1$. Therefore, the following resilience property holds just like in the Gaussian case.
\begin{propo}[Resilience condition for generalized linear regression]\label{lemma:glm-resilience}
With probability $1-\delta$ it holds that for all $Q\subset T$ with $|Q|\ge(1-2\eps)n$,
\begin{align*}
\frac{1}{n}\|\sum_{i\in Q} \left(y_i-b'({\beta^*}^\top\x_i)\right) \x_i\|\le  \Theta(\eps\log(1/\eps)).
\end{align*}
and for all $Q\subset T$ with $Q\le \eps n$,
\begin{align*}
\frac{1}{n}\sum_{i\in Q} y_i\le  \Theta(\eps\sqrt{\log(1/\eps)})\\
\frac{1}{n}\|\sum_{i\in Q} y_i\x_i\|\le  \Theta(\eps\log(1/\eps))
\end{align*}
\end{propo}
We conclude that 
\begin{align}
\frac{1}{n}\sum_{i\in \hat S \cap T} \left(y_i-b'({\beta^*}^\top\x_i)\right) ({\beta^*}-\hat\beta)^\top\x_i \le \eps\log(1/\eps)\|{\beta^*}-\hat\beta\|\label{eqn:glm-yx-resilience}
\end{align}
Next, note that by the lower bound on $b''(\cdot)$, it holds that
$$
\sum_{i\in (\hat S\cap T)}b''(\beta^\top\x_i)\x_i\x_i^\top \succeq \Theta(I),
$$
and hence by strong convexity
$$
\sum_{i\in \hat S \cap T} \left(b'({\beta^*}^\top\x_i) - b'(\hat\beta^\top\x_i) \right)({\beta^*}-\hat\beta)^\top\x_i \ge \Theta(\|\beta^*-\hat\beta\|^2)
$$
Together we get
$$
\sum_{i\in \hat S \cap T} (y_i-b'({\hat\beta}^\top\x_i)) (\beta-\hat\beta)^\top\x_i\ge \Theta(\|\beta^*-\hat\beta\|^2) + \Theta(\eps\log(1/\eps)\|{\beta^*}-\hat\beta\|)
$$
\textbf{Upper bound on the RHS}\\
\textbf{1. Upper bound on the negative log-likelihood.}

Define random variable $z_i = -\log f(y_i|\langle\beta^*, \x_i\rangle)$. 
Since $y_i-\E[y_i|\x_i]$ is a sub-Gaussian random variable, for a given $\x_i$, $\Pr(y_i-\E[y_i]\ge \sqrt{\log(1/\delta)})\le \delta$, thus 
\begin{align*}
\Pr(z_i\ge \log(1/\delta)) =& \Pr(z_i\ge \log(1/\delta)|y_i\le \sqrt{\log(1/\delta)}) + \Pr(y_i\ge \sqrt{\log(1/\delta)})\\
\le& \delta\sqrt{\log(1/\delta)}+\delta\\
\implies \Pr(z_i\ge t) \le e^{-t}\sqrt{t}\le e^{-t/2}.
\end{align*}
Since this is true regardless of $\x_i$, we have $z_i$ is a sub-exponential random variable. By the resilience property (Corollary~\ref{cor:sub-exponential-resilience}), we have that
\begin{align}
\frac{1}{n}\sum_{i\in T\setminus \hat S}-\log f(y_i|\langle\beta^*, \x_i\rangle) = \frac{1}{n}\sum_{i\in T\setminus \hat S} z_i\le \Omega(\eps\log(1/\eps)).\label{eqn:glm-RHS-neg-upper}
\end{align}

\textbf{2. Turn likelihood bound into square error bound}.
Let $\theta = \langle\beta^*, \x\rangle$, and $g(\cdot)=b'^{-1}(\cdot)$. Taking derivative of the negative likelihood function over the mean yields 
$\frac{b'(\theta)-y}{b''(\theta)}$. This implies for each $y$, $b'(\theta) = y$ minimize the negative log-likelihood function, and 
$$
\left(-\log f(y|\theta)\right) \ge \left(-\log f(y|g(y))\right) + (b'(\theta)-y)^2/C.
$$

Next we lower bound the likelihood of the minimizer
$$
-\log f(y|g(y)) = b(g(y))-yg(y)-\log(c(y))
$$
Since $b(0)=0$, and note that since $b''(\theta)\le C$ 
$$
b(0) \le b(\theta)-b'(\theta)\theta + C\theta^2.
$$
This implies
$$
-\log f(y|g(y)) \ge -Cg(y)^2-\log(c(y))
$$
Since $b'(x) \ge C_0 x$, it holds that $g(y)\le y/C_0$. This implies
\begin{align}
-\log f(y|g(y)) \ge -Cy^2/C_0-\log(c(y)).\label{eqn:glm-neg-lb}
\end{align}
Since after truncation $y_i$ are all bounded by $\Theta(\sqrt{\log(1/\eps)})$ from Proposition~\ref{lemma:glm-resilience}. We have
$$
\frac{1}{n}\sum_{i\in \hat S \cap E}-\log f(y_i|\langle\beta^*,\x_i\rangle) \ge \Theta(\frac{1}{n}\sum_{i\in \hat S \cap E}(y_i- b'({\beta^*}^\top\x_i))^2) - \Theta(\eps\log(1/\eps)).
$$
Combining this with the negative log-likelihood upper bound proved for $T\setminus \hat S$, we have
$$
\frac{1}{n}\sum_{i\in \hat S \cap E}(y_i- b'({\beta^*}^\top\x_i))^2 \le \Theta(\eps\log(1/\eps))
$$

By Cauchy Schwaz
\begin{align}
- \frac{1}{n}\sum_{i\in \hat S \cap E} (y_i-b'({\hat\beta}^\top\x_i))(\beta^*-\hat\beta)^\top\x_i\le& \sqrt{\frac{1}{n}\sum_{i\in \hat S \cap E} (y_i-b({\hat\beta}^\top\x_i))^2}\sqrt{\frac{1}{n}\sum_{i\in \hat S \cap E} ((\beta^*-\hat\beta)^\top\x_i)^2}\nonumber\\
\lesssim &\eps \log(1/\eps)\|\beta^*-\hat\beta\|\label{eqn:glm-upper-bound}
\end{align}
Combining Equation~\ref{eqn:glm-upper-bound} with the lower bound yields
$$
\|\beta^*-\hat\beta\| =O( \eps \log(1/\eps)).
$$
\end{proof}
\subsection{Convergence analysis of the alternating minimization algorithm}
% $$
% f(y) \le f(x)+\langle\nabla f(x) , y-x\rangle + \max_{z= \lambda x+ (1-\lambda)y, \lambda\in [0, 1]}(y-x)^\top \nabla^2f(z)(y-x)
% $$
\begin{propo}\label{prop:opt-smooth}
Suppose $\langle\nabla f(\hat\beta), \frac{\beta^*-\hat\beta}{\|\beta^*-\hat\beta\|}\rangle = \Delta$, and $\forall \beta, \frac{1}{\|\beta^*-\hat\beta\|^2}(\beta^*-\hat\beta)^\top\nabla^2 f(\beta) (\beta^*-\hat\beta) \le H$. 
There exists a point $\beta$ such that
$$
f(\beta) \le f(\hat\beta) - \frac{\Delta^2}{4H}.
$$
\end{propo}

\begin{lemma}[Algorithm~\ref{alg:alt} finds an approximate stationary point for generalized linear model]\label{lemma:alt-smooth} Given a set of datapoints $S = \{\x_i, y_i\}_{i=1}^n$ generated by a generalized linear model with $\eps_c$-fraction of corruption. Assuming that $b''(\mu)\le C$ for the generalized linear model and $n=\Omega(\frac{d+\log(1/\delta)}{\eps^2})$, then with probability $1-\delta$, the output of Algorithm~\ref{alg:alt}, $\hat\beta$, is a $\max(2\sqrt{C\eta}, \frac{2\eta}{\|\beta^*-\hat\beta\|})$-approximate stationary point with input parameters $R\ge \|\beta^*\|$, $\eps=\eps_c$. In particular, when $R = \infty$, the algorithm returns a $2\sqrt{C\eta}$-approximate stationary point.
\end{lemma}
\begin{proof}
Given that Algorithm~\ref{alg:alt} returns $\hat\beta$, and the correponding set is $\hat{S}$. Let $f(\beta)$ be the negative log likelihood function. Then $\nabla^2f(\beta) = \frac{1}{n}\sum_{i\in \hat{S}}b''(\beta^\top x_i)\x_i\x_i^\top\preceq C\cdot I$. Therefore, $\frac{1}{\|\beta^*-\hat\beta\|^2}(\beta^*-\hat\beta)^\top\nabla^2 f(\beta) (\beta^*-\hat\beta)\le C$. Applying Proposition~\ref{prop:opt-smooth}, it holds that there exists $\beta = \frac{\Delta}{2C}\frac{\beta^*-\hat\beta}{\|\beta^*-\hat\beta\|} + \hat\beta$ such that

$$
f(\beta) \le f(\hat\beta) - \frac{\Delta^2}{4C}.
$$

When $R = \infty$, since the algorithm terminate when it is impossible to make an $\eta$ improvement over the current point $\hat\beta$, it holds that
$$
|\Delta| \le 2\sqrt{C\eta},
$$
and $\hat\beta$ is a $\sqrt{2C\eta}$-approximate stationary point.

When $R$ only satisfies $\|\beta^*\|\le R$, we need to make sure $\beta$ is a convex combination of $\beta^*$ and $\hat\beta$ to make sure it is a valid solution for the optimization problem. Therefore, if $\Delta/2C\ge \|\beta^*-\hat\beta\|$, we have
$$
f(\beta^*)\le f(\hat\beta)-\Delta\|\beta^*-\hat\beta\|/2
$$
which implies 
$$
\Delta\le \frac{2\eta}{\|\beta^*-\hat\beta\|}
$$
\end{proof}
Since $b''(\theta) = \exp(\theta)$ is unbounded for Poisson regression, the above analysis does not apply, and here we prove our alternating minimization algorithm still returns an approximate stationary point
\begin{lemma}[Algorithm~\ref{alg:alt} finds an approximate stationary point for Poisson regression (Restatement of Lemma~\ref{lemma:alt-poisson-main})]\label{lemma:alt-poisson}
Given a set of datapoints $S = \{\x_i, y_i\}_{i=1}^n$ generated by a Poisson model with $\eps_c$-fraction of corruption. Assuming that $n=\Omega(\frac{d+\log(1/\delta)}{\eps^2})$, then with probability $1-\delta$, the output of Algorithm~\ref{alg:alt} with input parameters $\eps=2\eps_c, R\ge \|\beta^*\|, \eta = \eps^2/(dn)$, is a $\max(\eps, \frac{2\eps^2}{\|\beta^*-\hat\beta\|})$-approximate stationary point. \end{lemma}
\begin{proof}
Define
$$
H = \frac{1}{n\|\beta^*-\hat\beta\|^2}\sum_{i\in \hat{S}} \exp(\hat\beta^\top\x_i)((\beta^*-\hat\beta)^\top\x_i)^2,
$$
to be the second order derivative along the $\beta^* - \hat\beta$ direction. For every point $\beta = (1-\lambda)\hat\beta+\lambda\beta^*, 0 \le \lambda\le 1$, the second order derivative is 
\begin{align*}
\frac{1}{n\|\beta^*-\hat\beta\|^2}\sum_{i\in\hat{S}} \exp(((1-\lambda)\hat\beta+\lambda\beta^*)^\top\x_i)((\beta^*-\hat\beta)^\top\x_i)^2\\
\le \frac{1}{n\|\beta^*-\hat\beta\|^2}\left(\sum_{i\in\hat{S}} \exp(\hat\beta^\top\x_i)((\hat\beta-\beta^*)^\top\x_i)^2 + \sum_{i \in \hat{S}}\exp({\beta^*}^\top\x_i)((\beta^*-\hat\beta)^\top\x_i)^2\right)\\
\le H + \frac{1}{n\|\beta^*-\hat\beta\|^2}\sum_{i \in \hat{S}}\exp({\beta^*}^\top\x_i)((\beta^*-\hat\beta)^\top\x_i)^2
\end{align*}
Since ${\beta^*}^\top\x_i$ is sub-Gaussian, it holds that $\max_{i\in \hat{S}}{\beta^*}^\top\x_i\le \sqrt{\log(n)}$ with probability $0.99$. Hence with probability $0.99$, we have
$$
\le H + \exp(\sqrt{\log(n)})
$$

Consider two scenarios for the value of $H$.\\
\textit{1. $H\le C\cdot dn$.}

We can apply the same argument as in Lemma~\ref{lemma:alt-smooth} and obtain that $\hat\beta$ is a $\max(2\sqrt{dn\eta}, \frac{2\eta}{\|\beta^*-\hat\beta\|})$ approximate stationary point. Plugging in that $\eta=\eps^2/(dn)$, we get that $\hat\beta$ is a $\max(\eps, \frac{\eps^2}{\|\beta^*-\hat\beta\|})$ approximate stationary point.

\textit{2. $H\ge C\cdot dn$.} 

We will show this can not happen due to the termination condition of our algorithm. Due to the norm concentration of sub-Gaussian random vector, with probability $0.99$, it holds that $\max_{i\in \hat{S}}\|\x_i\|^2\le d+\Theta(\sqrt{d})$. Denote $D = \sum_i\exp(\hat\beta^\top\x_i)$. Then
%Since $\|\hat\beta\|, \|\beta^*\|=O(1), \|\x_i\|=O(\sqrt{d})$, it is easy to see that
$$
D = \frac{1}{n}\sum_{i\in \hat{S}} \exp(\hat\beta^\top\x_i) \ge H/\max_{i\in \hat{S}}\|\x_i\|^2\ge \Theta(H/d).
$$
The following inequality always holds (log sum inequality)
$$
\frac{1}{n}\sum_{i\in \hat{S}} \exp(\hat\beta^\top\x_i)(\hat\beta^\top\x_i) \ge  D\log(D).
$$
The following inequality holds since ${\beta^*}^\top \x_i\le \sqrt{\log (n)}$
$$
\frac{1}{n}\sum_{i}\exp(\hat\beta^\top\x_i)({\beta^*}^\top\x_i) \le D\sqrt{\log(n)}
$$
Therefore, 
$$
\frac{1}{n}\sum_i\exp(\hat\beta^\top\x_i)(\hat\beta^\top-\beta^*)^\top\x_i\ge D(\log(D)-\sqrt{\log n})\ge \frac{H}{\Theta(d)}(\log (H/\Theta(d))-\sqrt{\log n})
$$
Since $H\ge \Theta(dn)$, it holds that 
$$
\Delta := \frac{1}{n}\sum_i\exp(\hat\beta^\top\x_i)(\hat\beta^\top-\beta^*)^\top\x_i\ge  \Theta(\frac{H\sqrt{\log n}}{d})
$$
and that the gradient satisfies of the 
\begin{align*}
&\frac{1}{n}\sum_{i\in \hat{S}}(\exp(\hat\beta^\top\x_i)-y_i)(\hat\beta^\top-\beta^*)^\top\x_i\\
\ge&  \frac{H}{d} - \exp(\sqrt{\log (1/\eps)})\sqrt{d} \\
\ge& \Theta(H/d),
\end{align*}
where the first inequality holds since $y_i\le \exp(\sqrt{\log(1/\eps)})$, $\|\x_i\| = \Theta(\sqrt{d})$, and the second inequality holds since $n \ge 2\sqrt{d}/\eps \ge 2\exp(\sqrt{\log (1/\eps)})\sqrt{d}$.

Note that in this regime the second order derivative along the $\beta^*-\hat\beta$ direction is upper bounded by $2H$ for every point $\beta = (1-\lambda)\hat\beta+\lambda\beta^*, 0 \le \lambda\le 1$. Therefore, for every point in $\beta_\lambda = (1-\lambda)\hat\beta+\lambda\beta^*$, the value of the likelihood function is bounded as 
$$
-\frac{1}{n}\sum_{i\in \hat{S}}\log f(y_i|\langle\beta, \x_i\rangle) \le -\frac{1}{n}\sum_{i\in \hat{S}}\log f(y_i|\langle\hat\beta, \x_i\rangle) - \lambda \Delta + H \lambda^2
$$
If $\Delta\le 2H$, setting $\lambda = \Delta/2H$ yield
\begin{align*}
-\frac{1}{n}\sum_{i\in \hat{S}}\log f(y_i|\langle\beta, \x_i\rangle) \le -\frac{1}{n}\sum_{i\in \hat{S}}\log f(y_i|\langle\hat\beta, \x_i\rangle)-\Delta^2/4H.
\end{align*}

Since $\Delta\gtrsim H/d$, and $H\ge nd$, the drop in objective value $\Delta^2/4H = \Omega(n/d) \ge \Omega(1/\eps^2)$.

If $\Delta\ge 2H$, setting $\lambda = 1$ yield
\begin{align*}
-\frac{1}{n}\sum_{i\in \hat{S}}\log f(y_i|\langle\beta, \x_i\rangle) \le -\frac{1}{n}\sum_{i\in \hat{S}}\log f(y_i|\langle\hat\beta, \x_i\rangle)-H,
\end{align*}
and the drop in objective value is $H\ge nd >> 1/\eps^2$. Since $\eta\le 1/\eps^2$, we conclude that $H$ must be smaller than $C\cdot dn$ when the algorithm terminate.

\end{proof}

% \begin{propo}
% Suppose $\langle\nabla f(\hat\beta), \frac{\beta^*-\hat\beta}{\|\beta^*-\hat\beta\|}\rangle = \Delta$, and $\frac{1}{\|\beta^*-\hat\beta\|^2}(\beta^*-\hat\beta)^\top\nabla^2 f(\beta) (\beta^*-\hat\beta) \le H, \forall \beta = \lambda \beta^* + (1-\lambda)\hat\beta$. 
% If $H\ge - \Delta/2$, there exists a point $\beta = \lambda \beta^* + (1-\lambda)\hat\beta$ such that
% $$
% f(\beta) \le f(\hat\beta) - \frac{\Delta^2}{2H}.
% $$
% If $H< \Delta/2$, there exists a point $\beta = \lambda \beta^* + (1-\lambda)\hat\beta$ such that
% $$
% f(\beta) \le f(\hat\beta) + \Delta/2 .
% $$
% \end{propo}

\section{Handling sample corruption model}\label{sec:sample-corruption-model-appendix}
This section we discuss the algorithm for the sample corruption model, where the adversary is allowed to change the covariate $\x_i$ in addition to label $y_i$. At the end of Algorithm 4 in~\cite{dong2019quantum}, with high probability, it will return a set of $\eps$ corrupted data points that satisfies $\|\frac{1}{n}\sum_{i\in \hat{S}}\x_i\x_i^\top - I\| =  O(\eps\log(\eps))$. Corollary~\ref{prop:corrupt} then implies that $\|\frac{1}{n}\sum_{i\in \hat{S}\cap E}\x_i\x_i^\top\| =  O(\eps\log(\eps))$. The rest of the proof proceeds as in the label corruption setting.
\section{Non-identity covariance}\label{sec:non-id-cov}
Our result of Theorem~\ref{thm:gaussian-label}, Theorem~\ref{thm:poisson-main}, Theorem~\ref{thm:binomial-label} apply to the case where the covariance of $\x_i$ is identity. Here we argue that the result holds for general covariance $\Sigma$ where the guarantee is in terms of $\|\hat\beta-\beta^*\|_\Sigma$, which is the root-mean-square error in the Gaussian setting. First note that Algorithm~\ref{alg:alt} on input $\left(S={(\x_1,y_1),\ldots, (\x_n, y_n)}, \eps, \eta, R\right)$ with non-identity covariance $\Sigma$ output $\hat\beta$ if and only if running it on input $S={(\Sigma^{-1/2}\x_1,y_1),\ldots, (\Sigma^{-1/2}\x_n, y_n)}, \eps, \eta$ with the constraint $\|\Sigma^{-1/2}\beta\|_2\le R$ output $\Sigma^{1/2}\hat\beta$. 

Although the constraint $\|\Sigma^{-1/2}\beta\|_2\le R$ is different from the $\ell_2$ constraint in the analysis of Algorithm~\ref{alg:alt}, with slightly different arguments as in Theorem~\ref{thm:gaussian-label}, Theorem~\ref{thm:poisson-main}, Theorem~\ref{thm:binomial-label}, we can show that running Algorithm~\ref{alg:alt} on $S={(\Sigma^{-1/2}\x_1,y_1),\ldots, (\Sigma^{-1/2}\x_n, y_n)}, \eps, \eta$ with the constraint $\|\Sigma^{-1/2}\beta\|_2\le R$ will output $\hat\beta'$ such that $\|\hat\beta'-\Sigma^{1/2}\beta\|_2$ has the desired error rate. This implies Algorithm~\ref{alg:alt} running on $\left(S={(\x_1,y_1),\ldots, (\x_n, y_n)}, \eps, \eta, R\right)$ will output $\hat\beta = \Sigma^{-1/2}\hat\beta'$ such that $\|\hat\beta-\beta\|_\Sigma$ has the desired error rate.

\section{Auxiliary Lemmas}

\begin{propo}[Resilience of sub-Gaussian sample {\cite[Corollary 4]{jambulapati2020robust}}]\label{prop:resilience}
	 Let $G=\{\x_i\in \reals^d\}_{i=1}^n$ be a dataset satisfies Assumption~\ref{asmp:sub-gaussian-design}. For any $\eps\in [0, 1/2]$. If 
	 \begin{eqnarray}
	 	n=\Omega\left(\frac{d + \log(1/\delta)}{\eps^2}\right)\;,
	 \end{eqnarray} 
	 then with probability at least $1-\delta$ there exists an absolute constant $C>0$ such that for any subset $T\subset G$ and $|T|\geq (1-\eps)n$, we have
	\begin{eqnarray}
	\|\frac{1}{|T|}\sum_{i\in T}\x_i\| \le C\eps\sqrt{\log(1/\eps)}\\
    \|\left(\frac{1}{|T|}\sum_{i\in T}\x_i\x_i^\top\right) - I\| \le C\eps\log(1/\eps)
	\end{eqnarray}
	
	and that for any subset $T\subset G$ and $|T|\geq \eps n$, we have
	\begin{eqnarray}
	\|\frac{1}{|T|}\sum_{i\in T}\x_i\| \le C\sqrt{\log(1/\eps)}\\
    \|\left(\frac{1}{|T|}\sum_{i\in T}\x_i\x_i^\top\right) - I\| \le C\log(1/\eps)
	\end{eqnarray}

\end{propo}
\begin{coro}[Resilience of $1$-d sub-exponential sample]\label{cor:sub-exponential-resilience}
Let $G=\{x_i\in \reals^d\}_{i=1}^n$ contains i.i.d. samples drawn from $1$ sub-exponential distribution with zero mean. For any $\eps\in [0, 1/2]$. If 
	 \begin{eqnarray}
	 	n=\Omega\left(\frac{\log(1/\delta)}{\eps^2}\right)\;,
	 \end{eqnarray} 
	 then with probability at least $1-\delta$ there exists an absolute constant $C>0$ such that for any subset $T\subset G$ and $|T|\geq \eps n$, we have
	\begin{eqnarray}
	|\frac{1}{|T|}\sum_{i\in T}\x_i| \le C{\log(1/\eps)}
	\end{eqnarray}
\end{coro}

\begin{coro}[Resilience of corrupted sample with small covariance]\label{prop:corrupt}
Let $G=\{\x_i\in \reals^d\}_{i=1}^n$ be a dataset satisfies the condition in Proposition~\ref{prop:resilience} with $\Sigma = I$ and parameter $\eps$. Let $S = G\setminus L \cup E = T\cup E$ be obtained from set $G$ by corrupting $\eps n$ samples arbitrarily. There exists an absolute constant $C>0$ such that if 
$$
\|\frac{1}{n}\sum_{i\in \hat{S}}\x_i\x_i^\top - I\| =  O(\eps\log(\eps))
$$ for $\hat{S}$ with $\hat{S}\ge (1-\eps)n$, we have
	\begin{eqnarray*}
	\frac{1}{n}\|\sum_{i\in \hat{S}\cap T}\x_i\| \le C\eps\sqrt{\log(1/\eps)}\\
    \|\left(\frac{1}{n} \sum_{i\in \hat{S}\cap T} \x_i\x_i^\top\right) - I\| \le C\eps\log(1/\eps)
	\end{eqnarray*}
and
	\begin{eqnarray*}
	\frac{1}{n}\|\sum_{i\in \hat{S}\cap E}\x_i\| \le C\eps\sqrt{\log(1/\eps)}\\
    \|\left(\frac{1}{n} \sum_{i\in \hat{S}\cap E} \x_i\x_i^\top\right)\| \le C\eps\log(1/\eps)
	\end{eqnarray*}
\end{coro}
\begin{proof}
The statement about $T$ follows directly from resilience of sub-Gaussian sample. The statement about $E$ follows from the covariance bound and the condition of $T$.
\end{proof}

\begin{definition}[Approximate stationary point]\label{def:approx-sp}
Given a set of datapoints $\{\x_i, y_i\}_{i=1}^n$, probability density function $f(\cdot|\cdot)$, and $\beta^*\in \reals^d$, $\hat\beta\in \reals^d$. Let $\hat{S} = \argmin_{S:|S|=(1-\eps)n}\sum_{i\in S}-\log f(y_i|\langle{\hat\beta}, \x_i\rangle)$. We call $\hat\beta$ a $\eta$-stationary point if 
$$
\left(\nabla_\beta\sum_{i\in \hat{S}}-\log f(y_i|\langle \hat\beta, \x_i\rangle)\right)^\top \frac{(\hat\beta - \beta^*)}{\|\hat\beta - \beta^*\|} \le \eta
$$
\end{definition}

\begin{fact}[$k$-th moment bound of Poisson~\cite{ahle2022sharp}]\label{fact:k-moment-poisson}

Assuming that $y\sim \text{Poi}(\lambda)$, then $\E[y^k] = \max(k^k,\lambda^k)$ and $\E[(y-\lambda)^k] = \max(k^k,\lambda^k)$
% then $\E[\exp(t y)] = \exp(\lambda(e^t-1))$, $\E[\exp(t (y-\lambda))]= e$ when $t=\log(1+1/\lambda)$, and by the definition of sub-exponential $\|y_i\|_{\psi_1} = \max\{\lambda, 1\}$. Again, by the definition of sub-exponential rv,  $\E[(y-\lambda)^k]\le \max\{\lambda^k,1\}k^k $ 

\end{fact}
\begin{fact}\label{fact:log-factorial}
For all $n>0$,
$n\log n - n + 1\le \log n!\le (n+1)\log n - n + 1$
\end{fact}
% \begin{fact}
% For all $n\ge 1$,
% $n\log n - n + \log n + C + \frac{1}{12n+1}\le \log n!\le n\log n - n + \log n + C + \frac{1}{12n}$
% \end{fact}
\begin{fact}\label{fact:log-binom}
$$
 \log \binom{m}{y} \le - y\log(y/m) - (m-y)\log((m-y)/m) - \log \frac{y(m-y)}{m} + C_2
$$
\end{fact}
\end{document}